\documentclass{article}

\usepackage{fullpage}
\usepackage{microtype}
\usepackage{graphicx}
\usepackage{booktabs} 
\usepackage{amsmath, amssymb, amscd, amsthm, amsfonts}
\usepackage{graphicx}
\usepackage{hyperref}
\usepackage{booktabs}
\usepackage[dvipsnames]{xcolor}
\usepackage[square, numbers]{natbib}
\usepackage{titlesec}
\usepackage{tikz}
\usetikzlibrary{positioning}

\usepackage{subcaption}
\newtheorem{theorem}{Theorem}

\newtheorem{lemma}{Lemma}
\newtheorem{corollary}{Corollary}

\newtheorem{definition}{Definition}[section]

\begin{document}

\begin{center}

{\bf{\Large{Classification and Adversarial examples in an Overparameterized Linear Model: A Signal Processing Perspective}}}

\vspace*{.2in}

{\large{
\begin{tabular}{ccc}
  Adhyyan Narang$^{\dagger}$ & Vidya Muthukumar$^{\ddagger}$ & Anant Sahai$^{\dagger \dagger}$ \\
\end{tabular}}}

\vspace*{.2in}

\begin{tabular}{c}
Electrical and Computer Engineering, University of Washington$^{\dagger}$ \\
Electrical and Computer Engineering and Industrial and Systems Engineering, Georgia Institute of Technology$^{\ddagger}$ \\
Electrical Engineering and Computer Sciences, UC Berkeley $^{\dagger \dagger}$ \\
\end{tabular}

\vspace*{.2in}
\date{}

\end{center}




\begin{abstract}
  State-of-the-art deep learning classifiers are heavily overparameterized with respect to the amount of training examples and observed to generalize well on “clean” data, but be highly susceptible to infinitesmal adversarial perturbations.
  In this paper, we identify an overparameterized linear ensemble, that uses the ``lifted" Fourier feature map, that demonstrates both of these behaviors \footnote{These results were first presented at the ICML 2021 Workshop on Overparameterization: Pitfalls and Opportunities, June 2021.}.
  The input is one-dimensional, and the adversary is only allowed to perturb these inputs and not the non-linear features directly.
  We find that the learned model is susceptible to adversaries in an intermediate regime where classification generalizes but regression does not. Notably, the susceptibility arises despite the absence of model misspecification or label noise, which are commonly cited reasons for adversarial-susceptibility. These results are extended theoretically to a random-Fourier-sum setup that exhibits double-descent behavior.
  In both feature-setups, the adversarial vulnerability arises because of a phenomenon we term spatial localization: the predictions of the learned model are markedly more sensitive in the vicinity of training points than elsewhere. This sensitivity is a consequence of feature lifting and is reminiscent of Gibb's and Runge's phenomena from signal processing and functional analysis. Despite the adversarial susceptibility, we find that classification with these features can be easier than the more commonly studied “independent feature” models.
\end{abstract}
\newcommand\argmax{\operatorname*{\arg\max}}
\newcommand\argmin{\operatorname*{\arg\min}}
\newcommand{\appropto}{\mathrel{\vcenter{
  \offinterlineskip\halign{\hfil$##$\cr
    \propto\cr\noalign{\kern2pt}\sim\cr\noalign{\kern-2pt}}}}}

\newcommand\diag{\operatorname*{diag}}
\newcommand{\bigO}{\mathcal{O}}
\newcommand\sinc{\operatorname*{sinc}}
\newcommand\boldparfirst[1]{\noindent \textbf{#1}}
\newcommand\boldpar[1]{\vspace{1em}
\noindent \textbf{#1}}
\newcommand\numberthis{\addtocounter{equation}{1}\tag{\theequation}}

\newlength\mylen
\newcommand\myinput[1]{%
\settowidth\mylen{\KwIn{}}%
\setlength\hangindent{\mylen}%
\hspace*{\mylen}#1\\}
\newcommand{\Xaug}{X_{\mathsf{aug}}}
\newcommand{\Xval}{X_{\mathsf{val}}}
\newcommand{\yaug}{y_{\mathsf{aug}}}
\newcommand{\E}{\ensuremath{\mathbb{E}}}
\newcommand{\Prob}{\ensuremath{\mathbb{P}}}
\newcommand{\survival}{\ensuremath{\mathsf{SU}}}
\newcommand{\contamination}{\ensuremath{\mathsf{CN}}}

\newcommand{\B}{B}
\newcommand{\xtrain}{\boldsymbol{x}_{\mathsf{train}}}
\newcommand{\xtraini}{x_{\mathsf{train}, i}}
\newcommand{\xtest}{\boldsymbol{x}_{\mathsf{test}}}
\newcommand{\mtrain}{M_{\mathsf{train}}}
\newcommand{\strain}{S_{\mathsf{train}}}
\newcommand{\stest}{S_{\mathsf{test}}}
\newcommand{\ytrain}{\boldsymbol{y}_{\mathsf{train}}}
\newcommand{\ytraini}{y_{\mathsf{train}, i}}
\newcommand{\ztrain}{z_{\mathsf{train}}}
\newcommand{\phitrain}{\phi_{\mathsf{train}}}
\newcommand{\Atrain}{A_{\mathsf{train}}}
\newcommand{\Btrain}{\widehat{A}_{\mathsf{train}}}
\newcommand{\Atest}{A_{\mathsf{test}}}
\newcommand{\ALifted}{\mathsf{Lift}}
\newcommand{\minNorm}{\mathsf{Min L2}}
\newcommand{\ytest}{y_{\mathsf{test}}}
\newcommand{\ztest}{z_{\mathsf{test}}}
\newcommand{\phitest}{\phi_{\mathsf{test}}}
\newcommand{\phirfs}{\boldsymbol{\phi}^{\mathsf{RFS}}}
\newcommand{\phifour}{\boldsymbol{\phi}^{\mathsf{F}}}
\newcommand{\avec}{\phi}
\newcommand{\learnedfunc}{\widehat{f}}
\newcommand{\sgn}{\ensuremath{\mathsf{sgn}}}
\newcommand{\Sigmabold}{\ensuremath{\boldsymbol{\Sigma}}}
\newcommand{\alphabold}{\ensuremath{\boldsymbol{\alpha}}}
\newcommand{\betabold}{\ensuremath{\boldsymbol{\beta}}}
\newcommand{\coeffs}{\ensuremath{\boldsymbol{\widehat{\alpha}}}}
\newcommand{\coeffssimp}{\ensuremath{\boldsymbol{\widehat{\widehat{\alpha}}}}}
\newcommand{\coeffsRFS}{\ensuremath{\boldsymbol{\widehat{\beta}}}}
\newcommand{\alphaRFS}{\ensuremath{\boldsymbol{\widehat{\alpha}}^{\mathsf{RFS}}}}
\newcommand{\alphaRFSj}{\hat{\alpha}^{\mathsf{RFS}}_{j}}
\newcommand{\alphaRFSalias}{\hat{\alpha}^{\mathsf{RFS}}_{2kn}}
\newcommand{\coeffsRFSsimp}{\ensuremath{\coeffssimp^{\mathsf{RFS}}}}

\newcommand{\learnedRFS}{f_{\coeffsRFS}}

\newcommand{\smallcoeffs}{\ensuremath{\boldsymbol{\widehat{\xi}}_2}}
\newcommand{\scaledcoeffs}{\ensuremath{\boldsymbol{\widehat{\beta}}}}
\newcommand{\boldalpha}{\ensuremath{\boldsymbol{\alpha}}}
\newcommand{\boldlambda}{\ensuremath{\boldsymbol{\lambda}}}
\newcommand{\boldbeta}{\ensuremath{\boldsymbol{\beta}}}
\newcommand{\boldzeta}{\ensuremath{\boldsymbol{\zeta}}}
\newcommand{\boldxi}{\ensuremath{\boldsymbol{\xi}}}

\newcommand{\falphahat}{f_{\alphahat}}
\newcommand{\alphahat}{\widehat{\alpha}}

\newcommand{\pointwisemargin}{\gamma_{\mathsf{pw}}}
\newcommand{\nativemargin}{\ensuremath{\gamma_{\mathsf{native}}}}
\newcommand{\gradweightedmargin}{\ensuremath{\gamma_{\mathsf{GW}}}}

\newcommand{\what}{\widehat{w}_{svm}}
\newcommand{\re}{\mathbb{R}}
\newcommand{\R}{\mathbb{R}}

\newcommand{\lambar}{\vec{\lambda}}
\newcommand{\argmintwo}{\underset{\lambar, \Lambda}{argmin}}
\newcommand{\I}{\mathbb{I}}

\newcommand{\lambdabold}{\ensuremath{\mu}}
\newcommand{\Kmatrix}{\ensuremath{K}}
\newcommand{\invKmatrix}{\ensuremath{K}^{-1}}
\newcommand{\Ibold}{\ensuremath{\mathbf{I}}}
\newcommand{\Ebold}{\ensuremath{\mathbf{E}}}
\newcommand{\Ubold}{\ensuremath{\mathbf{U}}}
\newcommand{\Qbold}{\ensuremath{\mathbf{Q}}}
\newcommand{\Ybold}{\ensuremath{Y}}
\newcommand{\Jbold}{\ensuremath{\mathbf{J}}}

\newcommand{\JboldTruncated}{\ensuremath{\tilde{\mathbf{J}}}}

\newcommand{\lambdaoptsvm}{\ensuremath{\mu^*}}
\newcommand{\lambdaunconstrained}{\ensuremath{\widetilde{\mu}^*}}
\newcommand{\vecones}{\ensuremath{\mathbf{1}}}

\newcommand{\xtrue}{X_{true}}
\newcommand{\xtrash}{X_{alias}}
\newcommand{\wtrue}{w_{true}}
\newcommand{\wtrash}{w_{alias}}
\newcommand{\wtildetrash}{\widetilde{w}_{alias}}
\newcommand{\an}[1]{{\color{red} Adhyyan: #1}}
\newcommand{\vm}[1]{{\color{blue} Vidya: #1}}

\newcommand{\nativej}{\gamma_{native}^j}
\newcommand{\native}{\gamma_{native}}
\newcommand{\nativehatj}{\widehat{\gamma}_{native}^j}
\newcommand{\nativehat}{\widehat{\gamma}_{native}}
\newcommand{\nativehatmodel}{\widehat{\Gamma}_{native}}
\newcommand{\G}{\textbf{J}}
\newcommand{\Loss}{\mathcal{L}}
\newcommand{\aveci}{\ensuremath{\boldsymbol{\phi_i}}}
\newcommand{\phibold}{\ensuremath{\boldsymbol{\phi}}}
\newcommand{\Mbold}{\ensuremath{\mathbf{M}}}
\newcommand{\PertSet}{\mathcal{X}}
\newcommand{\learnedPredictor}{\falphahat}
\newcommand{\trueFunction}{\ensuremath{f_{true}}}
\newcommand{\norm}[1]{\left\lVert#1\right\rVert_2}

\newcommand{\specComplexity}{\mathbb{R_A}}

\newcommand{\classlossadv}{\ensuremath{\mathcal{C}}_{adv}}
\newcommand{\Xvec}{\ensuremath{\mathbf{X}}}
\newcommand{\XVecTilde}{\ensuremath{\tilde{x}}}

\newcommand{\sumjdk}{\ensuremath{\sum\limits_{j=1, j \neq \truek}^d}}
\newcommand{\regloss}{\ensuremath{\mathcal{R}}}
\newcommand{\classloss}{\ensuremath{\mathcal{C}}}
\newcommand{\reglossn}{\regloss(\alphahat_{\mathsf{real}};n)}
\newcommand{\classlossn}{\classloss(\alphahat_{\mathsf{binary}};n)}

\newcommand{\survivalb}{\ensuremath{\mathsf{SU}}}
\newcommand{\contaminationb}{\ensuremath{\mathsf{CN}}}
\newcommand{\survivalr}{\survivalb_r}
\newcommand{\contaminationr}{\contaminationb_r}
\newcommand{\survivaln}{\survival(1;n)}
\newcommand{\contaminationn}{\contamination(1;n)}
\newcommand{\survivalnreal}{\survivalb_r(1;n)}
\newcommand{\contaminationnreal}{\contaminationb_r(1;n)}
\newcommand{\wip}{\mathrm{w.p.}}
\newcommand{\truek}{\ensuremath{t}}

\section{Introduction}

Deep neural networks are often overparameterized and can achieve zero error on the training data, even if it is corrupted by arbitrary amounts of label noise~\cite{neyshabur2014search,zhang2016understanding}.
Thus, their state-of-the-art performance is not easily explained by traditional interpretations of statistical learning theory~\cite{zhang2016understanding}.
At the same time, this performance is also notably brittle: it sharply degrades when the input can be subject to small adversarial perturbations that are imperceptible to humans~\cite{goodfellow2014explaining}.
Recently, the benefits of such overparameterized models that \emph{interpolate} the training data have been empirically replicated in the simpler settings of kernel machines~\cite{belkin2018understand} and linear models~\cite{geiger2019jamming,belkin2019reconciling}.
We now have a much clearer understanding of the conditions under which these benefits of overparameterization arise, particularly for regression tasks~\cite{bartlett2020benign,belkin2019two,hastie2019surprises,mei2019generalization,mitra2019understanding,muthukumar2020harmless,nakkiran,xie2020weighted}.

In this context, it is natural to examine whether we can use the emerging framework of good generalization of interpolating models to identify, distill, and understand overparameterized scenarios that perform well with respect to test classification accuracy, but are brittle to infinitesmal amounts of adversarial corruption. However, this is far from a trivial task.
In general, the impact of an adversarial perturbation on the input data has a highly non-linear effect on the prediction decision, and is challenging to quantify precisely.
For example, the optimal adversarial perturbation typically involves solving a nonconvex optimization problem and is not always efficiently computable, let alone expressible in closed form. Accordingly, most existing precise mathematical characterizations of the adversarial risk, whether in underparameterized or overparameterized regimes, are focused on linear models \emph{using the original data as features}, which we henceforth call a \emph{linear-model-on-data}.
The assumption of linear-model-on-data is a significant simplification, as the worst-case perturbation can be expressed in closed form; more importantly, it abstracts away central components of even the simplest nonlinear models that are used in practice.
The overparameterized regime in particular admits another critical limitation of the linear-model-on-data assumption, which is that perturbations would be allowed in higher dimension than the number of samples, making the adversary unreasonably powerful.

It is therefore instructive to look for the simplest example of an overparameterized nonlinear model that exhibits the \emph{generalizable-but-brittle} phenomenon.
The central contribution of this paper is to provide the first precise characterization of classification and adversarial risk for the \emph{overparameterized-linear-model-on-features}.
To see the connection to this type of model and neural networks, we note that learning in neural networks is commonly described as involving two different kinds of learning.
During training, the early stages learn a good (typically high-dimensional) feature map; while the last stage learns a near-optimal linear model on this feature map.
Consequently, our model is most closely related to the setting where the first stage of neural network learning — feature learning — is randomized independently rather than learned from training data ~\cite{chizat2019lazy,jacot2018neural}.
Importantly, even though the model is now linear \emph{in the feature map}, it continues to be nonlinear \emph{on the input data}.

Crucially, the adversarial examples in our model are shown to arise in the infinite sample-size limit due to \emph{deficiencies in the estimation process} as a consequence of interpolation, even when there is no label noise\footnote{In pioneering work,~\cite{belkin2018overfitting} also uncover adversarial examples as a consequence of interpolation, but with different asymptotics ($n \to \infty$ for an arbitrarily small perturbation $\epsilon$) and for their results, the presence of non-zero label noise is key. Our results show that adversarial examples can arise from interpolation even without label noise.}.
The starting point for our main insight is contained in recent work~\cite{muthukumar2020classification} that showed that classification can ``generalize well'' in the sense of asymptotic consistency in certain high-dimensional regimes, despite poor comparative regression performance.
This discrepancy is, at a high level, attributed due to the relative forgiveness of the 0-1 test loss compared to the squared loss; thus, even though classification is successful, the original function is fitted poorly.
This leads to the natural hypothesis that this poor function fitting might lead to the proliferation of adversarial examples.
We formally prove this hypothesis in our paper under the simple toy model of a Fourier feature map on $1$-dimensional data.
This model allows us to clearly visualize the learned function and brings the central mechanism for the proliferation of adversarial examples into sharp focus.
Nevertheless, our findings do extend empirically to other classes of features and $>1$-dimensional data.
Our analysis is notable for being sharp with matching upper and lower bounds on the adversarial risk; this is in contrast to just showing the existence of adversarial examples for a learned-model, which is tantamount to providing a lower-bound on adversarial risk.

The estimation-centric explanation is in contrast to the majority of existing explanations for adversarial examples, which suggest that these are an unavoidable consequence of various aspects of the ML pipeline: (a) high dimensionality of input data~\cite{gilmer2018adversarial,fawzi2018adversarial,mahloujifar2019curse,shafahi2018adversarial} (b)  model misspecification~\cite{nakkiran2019adversarial} (c) label noise~\cite{fawzi2016robustness,ford2019adversarial,belkin2018overfitting} (d) larger sample complexity (\textit{upper bounds}) of adversarial robustness v.s.~regular generalization~\cite{schmidt2018adversarially,yin2019rademacher}.
To rule-out the above as possible causes for adversarial examples in our model, we intentionally study a setup with low-dimensional input-data and without misspecification or label-noise.
The proposed estimation-centric hypothesis is an optimistic one, since it suggests that we can alter parts of our learning process to learn adversarially-robust models.

\paragraph{Our contributions:}

In this paper, we identify simple examples of the overparameterized linear model that admit good generalization \textit{in binary classification tasks} on test data, but whose performance provably degrades in the presence of small adversarial perturbations.
Importantly, \emph{the regression task does not generalize well} for these examples (in a manner first made explicit by~\cite{muthukumar2020classification,chatterji2021finite} and built on since in subsequent work~\cite{wang2021benign,cao2021risk}), which is key to the proved separation between classification performance and adversarial robustness.
Our models invoke standard Fourier and polynomial feature liftings on one-dimensional input data, endowed with non-uniform weights on the features that change with the number of samples $(n)$ and the dimension of the feature map $(\B)$.
The estimator considered is the minimum-$\ell_2$-norm interpolator of the binary labels, or, equivalently, the minimum-Hilbert-norm interpolator of the equivalent reproducing kernel Hilbert space (RKHS) that is defined by this feature map.
We consider high-dimensional overparameterized regimes in which $\B \gg n$ and study asymptotics as $n \to \infty$.
Letting the feature scalings vary 
with $(n,d)$ enables us to explore settings that deviate from the traditional RKHS theory, in the sense that kernel ridge regression (KRR) need not generalize well; see, e.g. Chapter 13,~\cite{wainwright2019high}.
An informal summary of our results is described below:
\begin{enumerate}
\item We show that in contrast to the regression error, classification can ``generalize well" in our high-dimensional regimes, in the sense of statistical consistency, i.e. we show that the classification test error goes to $0$ as $n \to \infty$.
Further, this favorable performance of classification tasks can occurs in settings where the corresponding \emph{linear-model-on-data} would not generalize even in classification tasks~\cite{muthukumar2020classification}; see Corollary~\ref{cor:riskbilevel} for a discussion. This shows that lifted models can behave qualitatively differently from independent feature models when the underlying data is low-dimensional.
\item We precisely characterize the adversarial risk of classification tasks to infinitesmal perturbations of the order of $1/n$ \textit{on the underlying one-dimensional input data space}, and show a separation between classification and adversarial risk: while the classification risk will go to $0$ as $n \to \infty$, the adversarial risk can go to $1$.
In this model, adversarial examples occur as a fundamental consequence of the estimation/learning process. This happens despite the setup appearing safe along many other commonly explored axes --- there is neither label noise nor model misspecification; there are plenty of training samples; and the attacker cannot target the lifted features directly.
\item Finally, we extend our results to a setup that we introduce called \emph{random-Fourier-sum} features.
While the details of this model are contained in Section~\ref{sec:rfs}, we provide a brief description here.
Each component of the RFS feature model comprises of a weighted sum of sinusoids where the weights are \emph{random}; in fact, independent and identically distributed.
Notably, the number of RFS features is completely customizable and distinct from the number of Fourier frequencies used in the original Fourier feature map\footnote{We elaborate on this point in the description of the RFS model in Section~\ref{sec:rfs}. In the meantime, the high-level analogy that we recommend to the careful reader is between a reproducing kernel Hilbert space with weighted Fourier feature maps (the original model) and the infinite width limit of a wide linear model (the RFS model).}.
Moreover, this model exhibits double-descent behavior as a function of the number of RFS features.
We show that if the number of RFS features grows faster than the number of frequencies included in the Fourier model, both the classification and adversarial risk approach the risk of the original Fourier feature model.
This is a significantly more non-trivial result to show than the corresponding regression risk as both the classification and adversarial test error are discontinuous in their arguments.
We also show empirically that the complementary scenario in which the number of RFS features is smaller than the number of Fourier frequencies (but still larger than the number of training examples) admits qualitatively distinct behavior: while adversarial examples continue to proliferate, they are widely distributed across the training data domain and no longer ``spatially localized''.
\end{enumerate}

\begin{figure*}[bthp]
  \begin{subfigure}[]{0.32\linewidth}
    \centering
    \includegraphics[width=\textwidth]{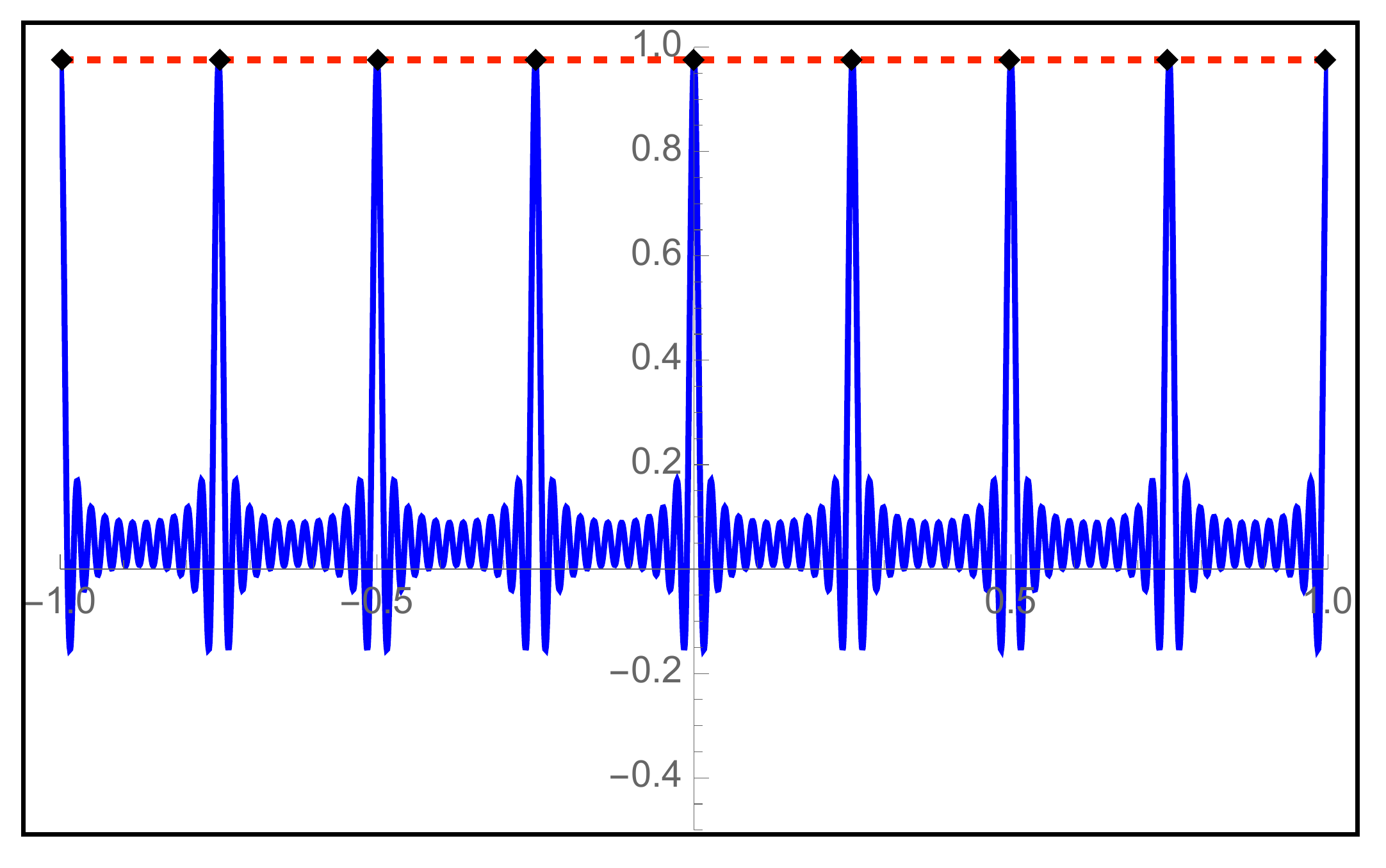}
  \caption{Overall recovered function $\hat{f}$.}
  \label{subfig:learned_function_total}
  \end{subfigure}
  \begin{subfigure}[]{0.32\linewidth}
    \centering
    \includegraphics[width=\textwidth]{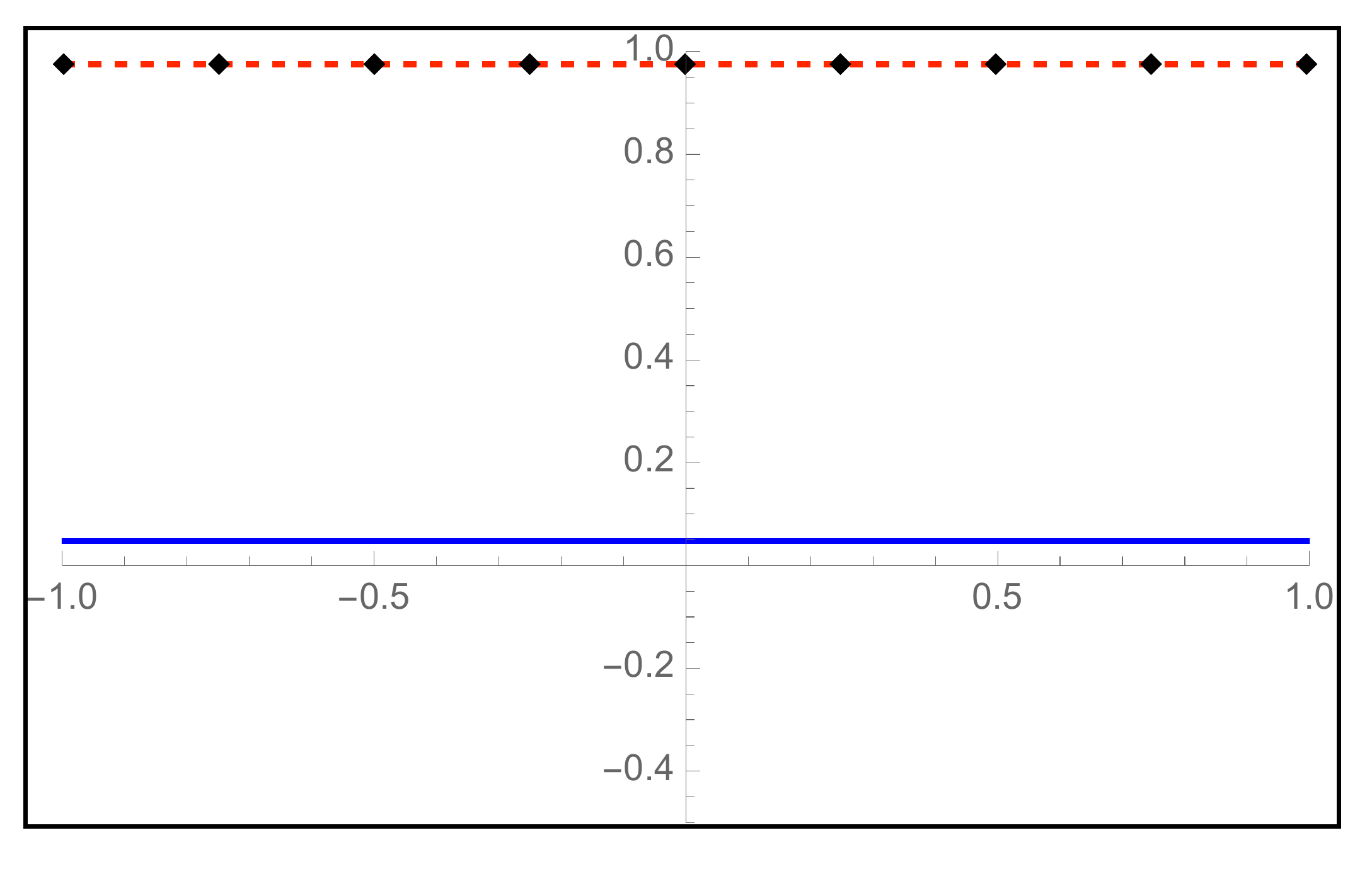}
  \caption{``True function" component of $\hat{f}$.}
  \label{subfig:learned_function_constant}
  \end{subfigure}
  \begin{subfigure}[]{0.32\linewidth}
  \begin{center}
    \includegraphics[width=\textwidth]{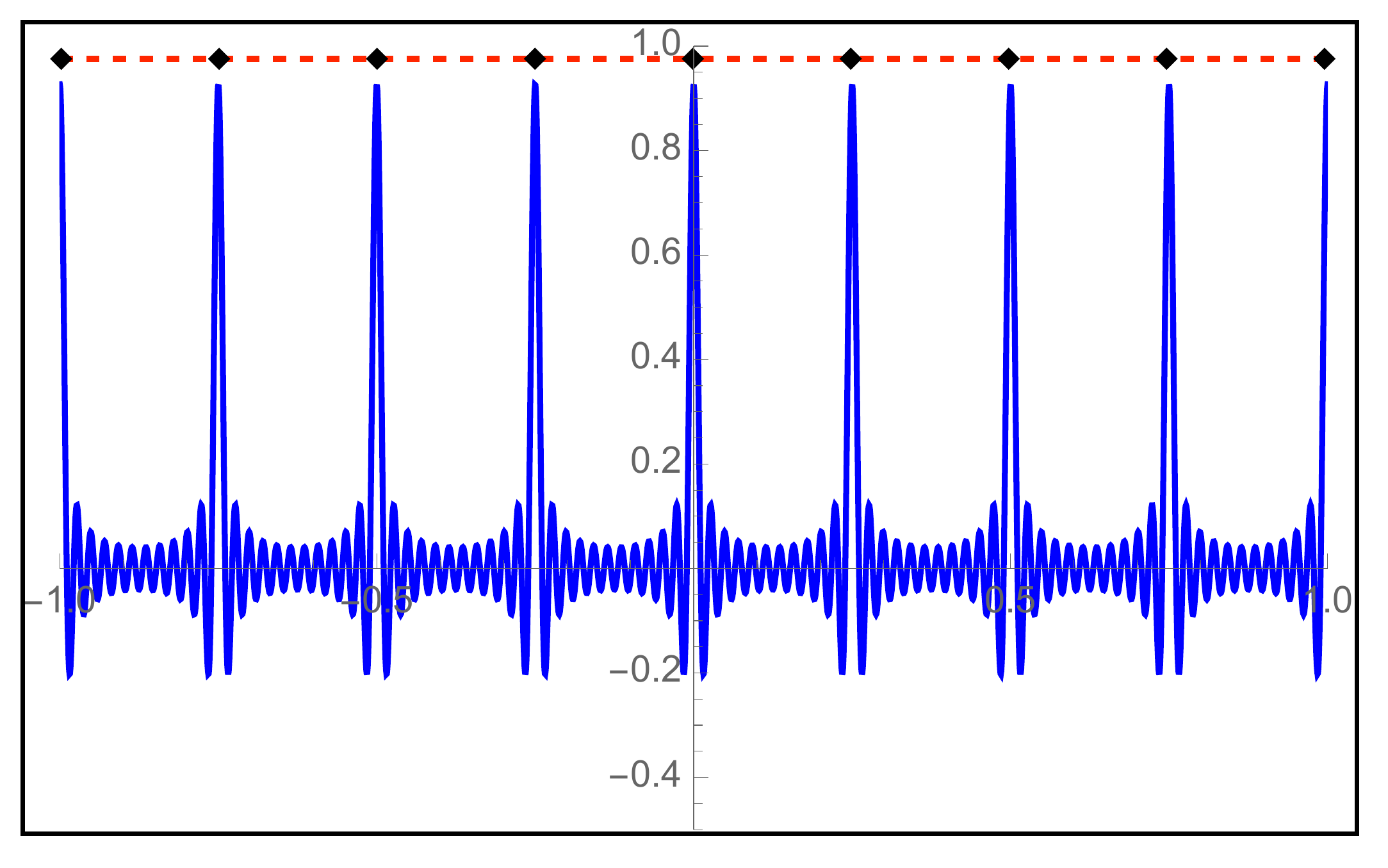}
  \end{center}
  \caption{``Aliased" component of $\hat{f}$.}
  \label{subfig:learned_function_aliases}
  \end{subfigure}
  \caption{Illustration of function recovered by minimum-$\ell_2$-norm interpolation with the Fourier feature map on 1-D data when the true function (and labels) are constant. Constant function and regularly spaced labels are pictured in dotted red and black respectively. Training data is regularly spaced in the interval $[-1,1]$.}
\end{figure*}\label{fig:illustrationresult}

\paragraph{Mechanism for results: Brittleness of classification test loss and Gibbs' phenomenon}

Our analysis uncovers the aforementioned results as direct consequences of the classical signal-processing phenomena of \textit{undersampling} and \textit{aliasing}.
Figure~\ref{fig:illustrationresult} showcases all of our results in an ``ultra-toy'' model.
Consider $n$ training samples that are regularly spaced on the interval $[-1,+1]$, and
labels generated by a true ``constant function'' $f^*(x) = +1$.
The model considered is the minimum-$\ell_2$-norm interpolator of this training data, using linear combinations of Fourier (sinusoid) features upto frequency $\B \gg n$ with a particular type of unequal scaling --- the constant feature is larger than the others, but only slightly.
Figure~\ref{subfig:learned_function_total} highlights that the ensuing recovered function $\hat{f}(\cdot)$ is very different from $f^*(\cdot)$; thus, implying poor regression performance.
However, the recovered function $\hat{f}(\cdot)$ hovers above $0$ \textit{almost everywhere}, and so will correctly predict binary labels $+1$ almost everywhere.
Finally, while classification is successful almost everywhere, the recovered function dips below $0$ at regular intervals of length $2/n$, yielding ubiquitous adversarial examples if the adversary can perturb inputs by up to $1/n$.

Underlying all of these results is the classical phenomenon of \textit{aliasing} that arises in statistical signal processing (SP), as a consequence of \textit{undersampling} the data with respect to the number of frequencies, i.e. $n \ll \B$. Undersampling is the SP counterpart of overparameterization.
Consider the set of functions that interpolates the labels: while this includes the true constant function $f^*(x)$, it also includes higher-frequency aliases of the form $f_{\ell}(x) := \cos(2\pi (\ell n) x)$ for values $\ell = 1,\ldots, \B/n$.
The minimum-$\ell_2$-norm interpolator can be precisely shown to select a function that includes a non-trivial proportion of all of these aliases.
Figure~\ref{subfig:learned_function_constant} shows the component of the recovered function consisting of the ``true" constant signal, and Figure~\ref{subfig:learned_function_aliases} shows the component of the recovered function comprising the higher-frequency aliases (which is centered at $0$).
Observe that the true signal is barely preserved --- its value is just above $0$.

This \textit{dissipation of signal} (which has been classically observed in statistical signal processing~\cite{chen2001atomic}) is the primary mechanism leading to poor regression performance.
The exact extent of dissipation is controlled by the relative scaling of the true constant feature and the aliases; hence, this scaling serves as a ``knob" that helps us control the type of inductive bias in the eventual solution.
In contrast to regression performance, average binary classification accuracy only depends on the recovered function remaining above $0$ almost everywhere.
In other words, the survived proportion of signal only needs to be enough to outweigh the \textit{contamination} effect of the aliases.
To this end, Figure~\ref{subfig:learned_function_aliases} shows a ``spatial localization" effect in the overall contamination by aliases: it is close to $0$ almost everywhere, but very large around the locations of the training points (and in the positive direction).
Thus, most of the contamination energy manifests in an \textit{asymptotically benign} manner, allowing for what we show is unusually good generalization in classification tasks.
Finally, the spatial localization effect also includes sharp ``contrarian'' dips in the contamination effect that manifest 
right next to the training  points.
Our analysis provides a precise interpretation of these dips through an explicit connection to Gibb's and Runge's phenomena in numerical analysis\footnote{Gibbs is a classical phenomena that manifests in the approximation of a discontinuous function by its truncated Fourier series leading  to ``overshoot'' peaks around the discontinuities. Spiritually, this happens here as the min-norm interpolation of samples effectively attempts to approximate a set of Dirac deltas at the training points. When the signal is sufficiently attenuated to be close to the decision boundary (in this case, the zero-line), the Gibbs-type overshoot spawns adversarial examples in the neighborhood of the training points.  
}, and shows that they are the cause of adversarial examples in our model.

\section{Related work}

\paragraph{Generalization:}
Under some conditions (outside of the more commonly studied neural-tangent-kernel (NTK) regime~\cite{jacot2018neural}), the outcome of the training process of an infinitely wide neural network can be interpreted as a data-dependent kernel~\cite{dou2020training,long2021properties}.
This process is believed to admit significant approximation-theoretic benefits~\cite{ghorbani2020neural}.
Consequently, the ensuing attractive test performance is not easily explained by data-dependent generalization bounds based on the Rademacher complexity and the training data margin\footnote{See, e.g.~\cite{bartlett2020failures} for a proof of failure of all such ``model-dependent" generalization bounds in overparameterized regimes.}.
Intriguing preliminary evidence exists that kernels that are too flexible for a regression task could generalize well in classification\footnote{Other intriguing and non-standard properties arise for these sufficiently high-dimensional linear models, such as the proliferation of support vectors and ensuing connection between the hard-margin SVM and the least-norm interpolation of discrete labels~\cite{hsu2020proliferation}.}.
This evidence is currently provided in high-dimensional linear models with independent sub-Gaussian and Gaussian features~\cite{muthukumar2020classification,chatterji2021finite,wang2021benign,cao2021risk}.
While these have been recently studied as a proxy for kernel methods, they comprise examples of the very different \emph{linear-model-on-data} paradigm.
Hence, these caricatures cannot be easily interpreted as feature maps of input data, nor can their precise mathematical analysis be easily extended to the more realistic \emph{linear-model-on-features} setting.

Perhaps one of the most intriguing open questions in this setup is whether the benefits of overparameterization and interpolation continue to hold for the \emph{linear-model-on-features} paradigm when the data dimension is held to a constant\footnote{Effectively modeling the dimension of real-world data in modern ML is an active area of research: for example, the ``manifold hypothesis'' posits that while the input dimension is often high, the effective dimension of the data is low.} with respect to the number of samples $n$.
In some cases, such as the proportional regime where the data dimension $p \propto n$, there is some universality in behavior for the linear-model-on-features and linear-model-on-data; see, e.g.~\cite{hastie2019surprises} for a recent discussion.
However, if $p$ is constant with respect to $n$, the picture is significantly more mixed and we cannot expect any universal behavior.
Indeed, examples such as the Laplace kernel~\cite{rakhlin2019consistency} demonstrate that interpolation need not generalize well here.
Thus, we need to consider stylized models for a tractable analysis.
The Fourier feature map represents one such stylized model where minimum-$\ell_2$-norm interpolation can be observed to generalize well even when the input data is $1$-dimensional in a manner that parallels the easier \emph{linear-model-on-data} paradigm~\cite{muthukumar2020harmless,muthukumar2020classification}.
It is a particularly attractive choice because of its use in designing measurement matrices used in the compressive sensing literature~\cite{candes2005decoding,candes2006robust}, as well as more classical connections between the idea of overparameterization and \emph{undersampling} of a high-frequency signal.
Our work leverages these connections and represents a first step towards a sharp characterization of the classification generalization error of this \emph{linear-model-on-features} using tools from signal processing and harmonic analysis.
Specifically, we show in this work that classification with the Fourier feature map can generalize well beyond the high-dimensional regimes considered in~\cite[Appendix A]{muthukumar2020classification}, settling one of the open questions that was raised there.
Consequently, classification is shown to be significantly easier than the corresponding linear-model-on-data case.

\paragraph{Adversarial examples:}
Several schools of thought suggest adversarial examples as an unavoidable consequence of various aspects of the ML pipeline: a) high dimensionality of input data~\cite{gilmer2018adversarial,fawzi2018adversarial,mahloujifar2019curse,shafahi2018adversarial}, b) model misspecification~\cite{nakkiran2019adversarial}; c) label noise~\cite{fawzi2016robustness,ford2019adversarial,belkin2018overfitting}, d) larger sample complexity(\textit{upper bounds}) of adversarial robustness v.s.~regular generalization~\cite{schmidt2018adversarially}.
This motivates a commonly conjectured fundamental tradeoff between regular accuracy and robustness.

A recent, alternative school of thought postulates that a ground-truth classifier that generalizes well \textit{and} is robust not only exists, but is contained within the model class.
According to this perspective, the inability of vanilla training procedures to find robust models is a consequence of suboptimalities in the \textit{estimation process}.
~\citet{raghunathan2020understanding} explore this possibility in the context of uncovering suboptimalities in the inductive bias of adversarial training \textit{for regression tasks}.
Their results have some potential connections with ours, but focus on the inductive bias of adversarial training rather than the originally trained model.
With regard to the latter,~\citet{ilyas2019adversarial} obtained an insightful characterization of ``informative-but-not-robust" features in state-of-the-art neural networks.
These features would generalize well, but be brittle to adversarial perturbations.
Overall, our work also supports the estimation-centric explanations for adversarial examples: indeed, several candidate models exist within our model family that would generalize well and be robust, including interpolating models. Our work has simple manifestations of all of these types of features: Figure~\ref{subfig:learned_function_constant} above shows an informative and robust feature, Figure~\ref{subfig:learned_function_aliases} shows uninformative features, and Figure~\ref{subfig:learned_function_total} shows an informative but not robust feature that arises as a aggregation of the two that is selected by minimum-$\ell_2$-norm interpolation.
Thus, these types of features arise as a consequence of our ordinary inductive bias and overparameterization\footnote{This perspective of \textit{under-specification} leading to brittleness in performance (along many additional axes besides adversarial robustness) is recently supported by extensive real-world experiments~\cite{d2020underspecification}.}.
This link between ``informative-but-not-robust'' features and minimum-$\ell_2$-norm interpolation in overparameterized models appears to be new; but is somewhat natural in light of the mechanism behind the generalization analysis that is provided in~\cite{muthukumar2020classification}.

From a mathematical standpoint, analyzing adversarial robustness in a precise and meaningful manner for overparameterized models is highly challenging.
An interesting recent line of work characterizes the worst-case (over the training data domain) or \textit{global} Lipschitz constant of a function class that interpolates training data, and proves a conjecture under fairly universal conditions that increased overparameterization can reduce this global Lipschitz constant~\cite{bubeck2020law,bubeck2021universal}.
However, a poor global Lipschitz constant only implies the presence of \textit{one} adversarial example, not many; thus, \textit{necessary} conditions for an adversarially robust model remain open.
As mentioned in the introduction, precise analyses of the adversarial error for the case of \emph{linear-model-on-data}~\cite{javanmard2020precise} do exist.
These analyses heavily leverage the closed-form expression for the adversarial perturbation and are not easily extendible to the more realistic case of linear-model-on-features.
Further, in the highly overparameterized regime, the linear-model-on-data grants the adversary an unrealistic level of power by allowing them to perturb many more degrees of freedom than the number of training examples.
This motivates our present study of adversarial examples on extremely low-dimensional data mapped onto high-dimensional features.
Here, analyzing the precise downstream effect of an adversarial perturbation is non-trivial largely owing to the non-linearity of most such feature maps.
Our analysis is notable for: a) being sharp with matching upper and lower bounds on the adversarial risk, and b) uncovering adversarial examples in infinite-sample size, as a consequence of interpolation, \textit{despite} the absence of high dimensionality in input data, model misspecification, and label noise.
We conduct our analysis in the highly stylized model of the Fourier feature map, which leaves many directions for future work open; we discuss these directions in brief in Section~\ref{sec:simulations}.

\section{Setup}
\label{sec:setup}

\paragraph{Data and feature map:} Let $\strain := (\xtrain, \ytrain) \in (\re \times \{-1,+1\})^{n}$ be our training dataset, where $\xtrain$ denotes the one-dimensional training data and $\ytrain$ denotes the binary labels.
The $i^{th}$ training point is denoted by $\xtraini \in \re$, and its corresponding label is denoted as $\ytraini \in \{-1, +1\}$.
For the purposes of a clean and easily interpretable theoretical analysis, we make an ``ultra-toy" assumption that the training data points are chosen deterministically on a grid, i.e. $\xtraini = -1 + 2i/n$.
Section~\ref{sec:simulations} provides evidence that our theoretical predictions carry over to the case of training data chosen uniformly at random, i.e. $\xtraini \sim \text{Unif}[-1,1]$.
As is standard, we will evaluate population test error over data also chosen $\sim \text{Unif}[-1,1]$.

For ease of exposition, we consider the labels to be generated by a constant function, i.e. $y = f^\ast(x) = 1$ regardless of the point $x$ at training and test time.
We note that our theory applies to any generative model of the form $y = \sgn(\sin(2\pi k x))$, where $k < n$ is an integer (under appropriate adjustments to the feature map).
The problem of learning a robust classifier is only well-defined when the true classifier is robust, i.e. does not change sign when the input is perturbed slightly. The constant function is a quintessential example of such a function, hence motivating the choice of $f^\ast(\cdot)$.

Our model class consists of linear combinations of the high-dimensional Fourier feature map $\boldsymbol{\phi}: \R \to \R^\B$, defined as $\boldsymbol{\phi}(x) = \left[\frac{1}{\sqrt{2}}, \sin(\pi x), \cos(\pi x), \ldots \sin(\frac{\B}{2} \pi x), \cos(\frac{\B}{2}\pi x)\right]$.
This constitutes a high-dimensional ``lifting" of one-dimensional data.
The feature map consists of the first $\B$ elements of the Fourier basis, which can be verified to be orthonormal with respect to the uniform distribution on $[-1,1]$.
It also includes the true constant feature that generates the labels; thus, the training data is well-specified under this model.
To study the effects of overparameterization and interpolation, we select $\B \gg n$.

The motivation for the feature-map is to study the linear-model-on-features case, which is a significant generalization of the linear-model-on-data setup. We choose the Fourier feature-map because it (and random variants) have recently been successfully studied to explain modern ML phenomena \cite{muthukumar2020harmless, rahimi2008random, belkin2019two}, and because it allows us to leverage rich insights from signal processing to contextualize our problem.
Our focus of study is on the minimum-$\ell_2$-norm interpolator,
\begin{definition}[Minimum-$\ell_2$-norm interpolator]
  \label{def:minl2}
The minimum-$\ell_2$-norm interpolator of labels on Fourier features with weighting $\Sigmabold$ is given by
\begin{align*}
  \coeffs_2 & := \argmin_{\boldsymbol{\alpha} \in \re^d} \| \Sigmabold^{-\frac{1}{2}} \boldsymbol{\alpha}\|_2 \\
	    & s.t \; \; \boldsymbol{\phi}(\xtrain) \boldsymbol{\alpha} = \ytrain,
\end{align*}
where $\boldsymbol{\phi}(\xtrain)$ denotes the $n \times d$ training data matrix.
\end{definition}
Note that in the above definition, we have chosen the convention so that $\coeffs_2$ denotes the coefficients on the \emph{unweighted} Fourier features; accordingly, to model the effect of non-uniform weighting, we consider the minimum-\emph{weighted}-$\ell_2$-norm interpolator as above.
Equivalently, we could have denoted coefficients on the \emph{weighted} version of the Fourier features, in which case we would simply consider the minimum-$\ell_2$-norm interpolator.
These solutions and their generalization error are clearly equivalent.

The minimum-$\ell_2$-interpolator is known to require an explicit inductive bias that favors the location of the true feature (in this case, the constant feature).
This is why the features are weighted \textit{non-uniformly}, with higher weights on candidate (true) features than others\footnote{This is in spirit connected to the concept of minimum-Hilbert-norm interpolation using smooth kernels with a fast eigenvalue decay.
For expositions on why this non-uniform weighting is required for good generalization of the $\ell_2$-inductive bias in high-dimensional linear regression, see the recent analyses of benign overfitting~\cite{bartlett2020benign,muthukumar2020harmless}.
}.
In Definition~\ref{def:minl2}, $\Sigmabold := \text{diag}(\lambda_1,\ldots,\lambda_\B)$ denotes a diagonal weighting matrix with non-negative and decreasing entries; the higher the weight $\lambda_i$ is on feature $i$, the less it is penalized in the objective function.
An explicit choice of the non-uniform weighting is a parameterized ``bilevel'' structure of one higher weight and one lower weight, formally defined below.

\begin{definition}[Bi-level ensemble$(n,p,q)$]
  \label{def:bilevel_covariance}
Set the parameters $\B := n^p$ and $\gamma = n^{-q}$.
Now, define $\Sigmabold(n,p,q)$ as the diagonal matrix with entries as:
\begin{align*}
  \lambda_j=  \begin{cases}
  \gamma \B, & j = 1\\
  \lambda_L := {\frac{(1 - \gamma)\B}{\B-1}}, & \text{otherwise}.
  \end{cases}
\end{align*}
For this ensemble, we will fix $(p,q)$ and study the evolution of various quantities as a function of $n$.
\end{definition}

Intuitively, the parameter $p$ dictates the rate of growth of dimension with $n$, and the parameter $q$ dictates the relative weighting of the features.
Roughly speaking, a smaller value of $q$ implies easier generalization as the relative weight on the true feature vs the aliases, $\lambda_1/\lambda_L$, is higher. When $q = 0$, then $\lambda_1 = d$ and  $\lambda_L = 0$, and when  $q = p$, the weight is isotropic across all features i.e  $\lambda_1 = \lambda_L$.
The bilevel ensemble is a member of a ubiquitous family of high-dimensional models, such as spiked covariance models~\cite{wang2017asymptotics} and weak (random) feature models~\cite{belkin2019two,mei2019generalization}.
See Appendix A,~\cite{muthukumar2020classification}, for a detailed discussion on the properties of the bilevel ensemble.

The value of $\coeffs_2$ is unique and has a well-known analytical expression when $\Sigmabold^{\frac{1}{2}} \boldsymbol{\phi}(\xtrain)$ is full-rank (as is typically the case).
Lemma~\ref{lem:coeffs}, a version of which initially appeared in~\cite{muthukumar2020harmless} shows that the ``bilevel" weighting on features together with regularly spaced training data admits a particularly simple and interpretable expression for the coefficients recovered by least-norm interpolation. For a fixed $p$, the classification and adversarial error is monotonically non-decreasing in $q$ : a higher value of $q$ will lead to less weight placed on the true feature and higher weight on the aliases, causing poorer function recovery and higher error.

\paragraph{Test evaluation:}
Let $\learnedfunc(x) = \langle \coeffs_2, \boldsymbol{\phi}(x) \rangle$ denote the learned function parameterized by the minimum-$\ell_2$-norm interpolator, denoted by $\coeffs_2$. Clearly, this function depends on $(n, \B, \Sigmabold)$.
Accordingly, sometimes we will specify these using a subscript, but leave the dependence implicit when it is clear from context.
For a learned function $f: \re \mapsto \{-1, +1\}$, define the classification risk as follows. (where $\mathbb{I}$ denotes the indicator function, and all expectations and probabilities are calculated with respect to $x \sim \text{Unif}[-1,1]$):
\begin{equation}
  \label{def:class_risk}
  \mathcal{C}(f) = \Prob \left[\sgn(f(x)) \neq \sgn(f^\ast(x))\right]
\end{equation}
Further, we define the adversarial classification risk as:
\begin{equation}
  \label{def:adv_risk}
  \mathcal{C}_{\mathsf{adv}}(f, \epsilon) = \E \left[\max_{\XVecTilde \in \PertSet(x, \epsilon)} \mathbb{I}\left[\sgn(f(\XVecTilde)) \neq \sgn(f^\ast(x))\right]\right]
\end{equation}
Since our data is $1$-dimensional, the choice of the perturbation set is unambiguous: we pick $ \PertSet(x, \epsilon) = \{\XVecTilde: |\XVecTilde - x| \leq \epsilon \}$.
Whenever unspecified, we use the standard Euclidean norm for vectors and the spectral-norm (maximum singular value) for matrices.

\section{Adversarial Classification vs Classification}
\label{sec:main_result}

The goal of this section is to show the existence of classifiers which achieve asymptotically zero classification error but an adversarial error of $1$: for such a classifier, a randomly chosen test-point will have a zero probability of being misclassified but there will always exist an allowable perturbation that can induce misclassification. 
In the context of the bilevel model that was defined in Definition~\ref{def:bilevel_covariance}, we pose the following question: for any fixed $p$, is there a value (or a range of values) of $q$ for which the resulting minimum-$\ell_2$-norm interpolator will display the generalizable-but-brittle behavior, even as $n \to \infty$?

Answering this question requires a precise characterization of both classification and adversarial risk. 
We adopt the following strategy to do this:
\begin{enumerate}
  \item Using ideas from signal processing, we obtain an exact
    expression of the weights $\coeffs$ learned by the minimum-$\ell_2$-norm
    interpolator in Definition~\eqref{def:minl2}.
  \item The ensuing function turns out to be periodic, and a single period of the function exactly corresponds to the Dirichlet kernel, which is commonly used in Fourier analysis \cite{muscalu2013classical}.
  We then calculate the envelope of this kernel and characterize the (approximate) locations of zero-crossings of the recovered function.
  \item We use the zero-crossings to characterize the length of each misclassification region in both the average and adversarial case, and hence the risk.
\end{enumerate}

\subsection{Key ingredients for analysis}

In general, sharply characterizing the adversarial risk in highly overparameterized models is challenging, and the existing analyses for regression and classification do not provide readily applicable tools for this purpose.
In this section, we show that techniques from Fourier analysis provide a clean characterization of both classification and adversarial risk.
The lemma below describes the structure of the learned coefficients $\coeffs$ arising from minimum-$\ell_2$-norm interpolation with Fourier features on regularly spaced training data.
This result first appeared in~\cite{muthukumar2020harmless} and is restated here for completeness.

\begin{lemma}[Learned coefficients]
  \label{lem:coeffs}
Consider the lifted Fourier featurization on $n$ regularly spaced training data points, constant labels, and weighting scheme $\{\lambda_j\}_{j=1}^\B$ defined as a function of $(n,\B)$ as in Definition~\ref{def:bilevel_covariance}.
Then, the weights learned by the minimum-$\ell_2$-norm interpolator \eqref{def:minl2} can be written in closed form as: \[
    \widehat{\alpha}_i = \begin{cases}
      \frac{\sqrt{2} \lambda_1}{\sqrt{2} \lambda_1 + 2 \lambda_L \left(\frac{\B -1}{2n} \right)}, & i = 1 \\
      \frac{2 \lambda_L}{\sqrt{2} \lambda_1 + 2 \lambda_L \left(\frac{\B - 1}{2n}\right)}, & \exists k: i = 2kn, \\
      0, & \text{otherwise}
    \end{cases}
    .\]
\end{lemma}
From the definition of our feature-map, this lemma implies that the learned function adopts the form

\begin{equation}
  \label{eq:form_of_learned}
\widehat{f}(x) = \frac{a}{\sqrt{2}} + b \left[ \sum_{j=1}^{N_A} \cos(j n\pi x) \right]
\end{equation}
In the above, $a = \widehat{\alpha}_1$ and $b = \widehat{\alpha}_i$ for all $i = 2n$. 
This calculation is a result of the spiritual connection to matched filtering in signal processing.
While this \textit{exact} analysis  requires training data to be regularly spaced, Section~\ref{sec:simulations} provides evidence to show that the same phenomena arise when training data is drawn uniformly at random.
Denote using $N_A = (\B-1)/2n$ the number of aliases of the constant function in the first $\B$ real Fourier features.
In this section, we sharply evaluate the classification and adversarial risk of $\learnedfunc(\cdot)$ with respect to the constant function as a function of the parameters $a,b,\B,n$. While the intention is to provide analysis for the risk of the min-$\ell_2$-interpolator of Definition~\ref{def:minl2}, the results directly apply to any trigonometric function of the form above.

Since the learned function is periodic with period $2/n$, the probability of misclassification in the intervals $[-1, 1]$ or $[-1/n, 1/n]$ is the same. For simplicity of exposition, we restrict our attention to the latter.
We then observe that we can write $1 + 2\sum_{j=1}^{N_A} \cos(j n \pi x) = D_{N_A}(n \pi x) := \frac{\sin((N_A + 1/2) n \pi x)}{\sin(n \pi x/2)}$, where $D_{N_A}(\cdot)$ denotes the Dirichlet kernel\footnote{This kernel
  is visually extremely similar to the more well-known sinc kernel, as
  can be seen in Figure~\ref{fig:dirichletkernel}.}, which is commonly used in Fourier analysis to show the convergence of various Fourier series approximations to a function~\cite{muscalu2013classical}.
From this, we get 
\begin{equation}
\label{eq:form_of_learned_kernel} 
\widehat{f}(x) = \frac{2a - \sqrt{2}b}{2\sqrt{2}} + \frac{b}{2} \cdot \frac{\sin((N_A + 1/2) n \pi x)}{\sin(n \pi x/2)}.
\end{equation}

\begin{figure}[htpb]
    \centering
    \includegraphics[width=0.4\textwidth]{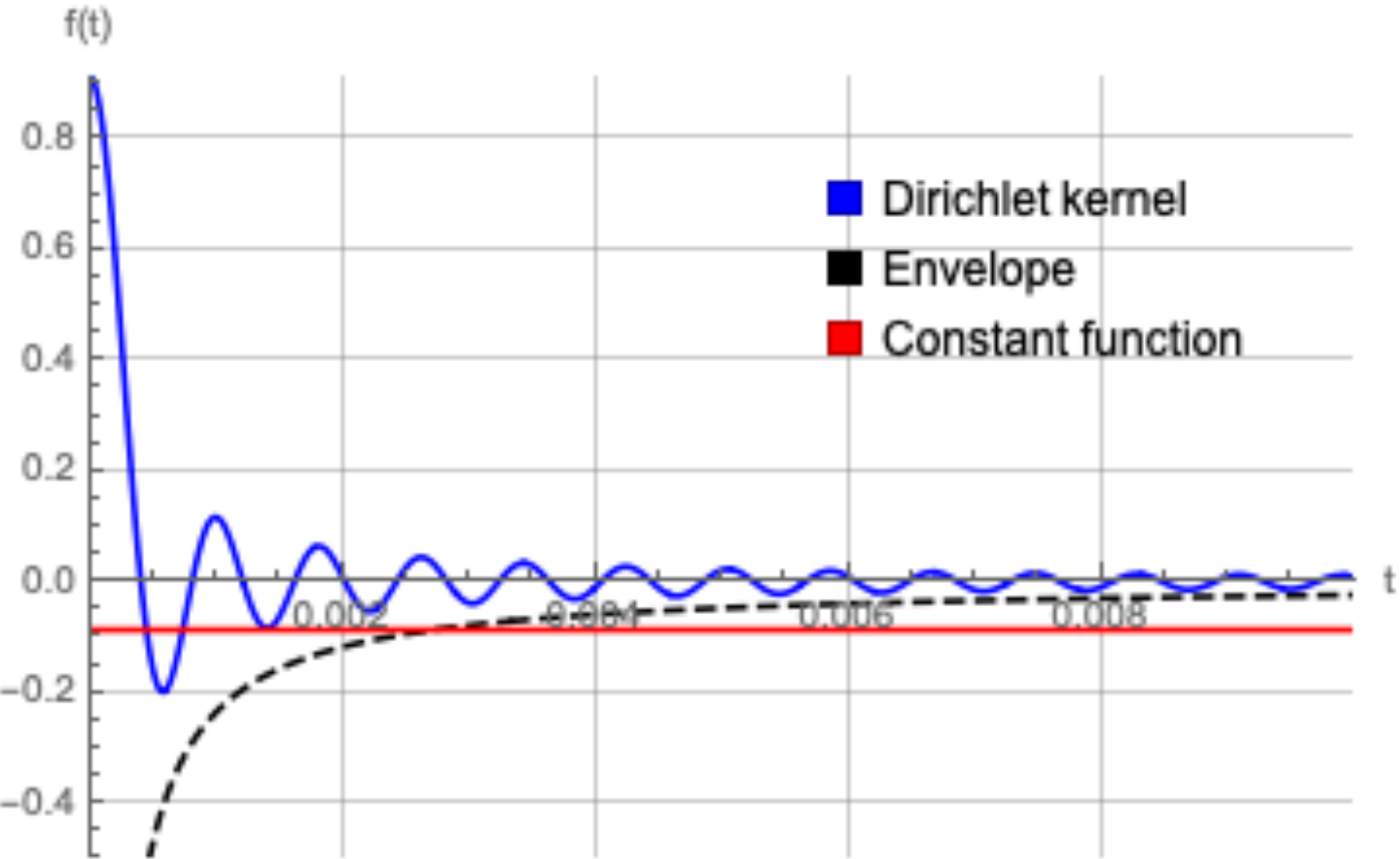}
  \caption{Depiction of the Dirichlet kernel component and its
    (negative) envelope as compared to the constant component of
    $\widehat{f}(x)$ for parameters $a = 0.13, b = 0.0055, \B = 4930, n =
    30$. Observe that $k^* = 2$ in this example.}
  \label{fig:dirichletkernel}
\end{figure}

A visualization of the Dirichlet kernel is shown in Figure~\ref{fig:dirichletkernel}.
Observe the presence of both positive and negative lobes of decaying amplitude.
It is well-known, e.g.~see~[Exercise 1.1]\cite{muscalu2013classical}, that the envelope of the Dirichlet kernel is given by $|D_N(\pi x)| \leq \frac{2}{\pi x} \text{ for all } N \geq 1$.
The envelope is depicted in Figure~\ref{fig:dirichletkernel}, and this characterization is critically used to sharply analyze both classification and adversarial risk.

\paragraph{From zero-crossings to risk.} Let $\mathcal{Z} = \{l_1, r_1 \ldots l_k, r_k\}$ be the roots of the learned function $\widehat{f}(\cdot)$ in sorted order; in the intervals $(l_i, r_i)$, the learned function's output is negative and in the intervals $(r_i, l_{i+1})$, its output is positive. Using this notation, the classification risk in the interval $(-1, 1)$ can be expressed as $\mathcal{C}(f) = \frac{1}{2} \sum_{i=1}^k (r_i - l_i)$.
By extending each of the misclassification intervals by $\epsilon$ on either side and accounting for overlaps between the intervals, define $\widetilde{r_i} = r_i + \epsilon$ and $\widetilde{l_i} = \max(l_i - \epsilon, \widetilde{r}_{i-1})$. It is easy to verify from the expressions that $\widetilde{r_i} \geq \widetilde{l_i} \geq \widetilde{r}_{i-1}$; this implies that there are no overlaps between the intervals $\{\widetilde{l_i}, \widetilde{r_i} \}_{i=1}^K$. Denote $\widetilde{\mathcal{Z}} = \{\widetilde{l_1}, \widetilde{r_1}, \ldots \widetilde{l_K}, \widetilde{r_K} \}$.
The adversarial classification risk can be similarly evaluated as $\mathcal{C}_{\mathsf{adv}}(f, \epsilon) = \frac{1}{2}\sum_{i=1}^K (\tilde{r_i} - \tilde{l_i})$.
It is easy to see that $\mathcal{C}_{\mathsf{adv}}(f,\epsilon)$ is increasing in $\epsilon$.

Our main results will encapsulate the following:
\begin{enumerate}
\item Lower bounding the adversarial risk for $\epsilon = 2/n$.
\item Upper bounding the classification risk (as it turns out, by the adversarial risk $\mathcal{C}_{\mathsf{adv}}(f,2\pi/\B)$).
\end{enumerate}


For both cases, it will be beneficial to consider each negative lobe
of the Dirichlet kernel seperately. We first prove an intermediate
result about the number of lobes that induce zero-crossings for the
learned function. Define the set of the indices of definitely benign
lobes as $S_{bl} = \{k: f(x) > 0 \; \; \forall x \in \left[\frac{2k-1}{h(\B,n)}, \frac{2k}{h(\B,n)} \right] \}$, where we denote $h(\B,n) := (\B - 1 + n)/2$ as shorthand.
Let the first lobe that does not induce a misclassification be denoted by $k^* = \min_{k \in S_{bl}} k$.
Obtaining a handle on this quantity in terms of the parameters of our setup will aid us in obtaining expressions for the risk.
For instance, for the choice of parameterization $(a,b,d,n)$ in Figure~\ref{fig:dirichletkernel}, we see that $k^{*} = 2$.

We show in the following lemma that the number of ``zero-crossing" lobes $k^*$ can be both upper and lower bounded.
Lemma~\ref{lem:k_star} is proved in Appendix~\ref{sec:kstarproof}.
\begin{lemma}[Expressions for $k^*$]
  \label{lem:k_star}

As defined above, we can characterize
\begin{align}\label{eq:kstarupperbound}
    k^{*} \leq \frac{2\sqrt{2} \cdot b \cdot (\B - 1 + n) + n(2a - \sqrt{2}b)}{2 \pi n (2a - \sqrt{2} b)} .
\end{align}
We also get $k^* \geq 1$ provided that
\begin{align}\label{eq:kstarlowerbound}
0.2172 \sqrt{2} b \cdot (\B - 1 + n) > n (2a - \sqrt{2}b) ,
\end{align}
\end{lemma}

In what follows, we will use the quantity $k^*$ as a crucial relative indicator of both classification and adversarial risk.

\subsection{Expressions for classification and adversarial risk}

With the above ingredients, we can now characterize the classification and adversarial risk.
Our first result characterizes the classification and adversarial risk for all learned functions of the form $\widehat{f}(x)$.

\begin{theorem}\label{thm:riskabdn}
  For the min-$\ell_2$-interpolator (Definition~\ref{def:minl2}) on Fourier features $\widehat{f}(x)$, the classification $0-1$ error evaluated with respect to the ground-truth constant function is upper bounded by
\begin{align}\label{eq:risk_abdn}
  \mathcal{C}(\widehat{f}) &\leq \frac{\sqrt{2} b \cdot h(\B,n)}{\pi h(\B,n) (2a - \sqrt{2}b)} + \frac{n}{2h(\B,n)} .
\end{align}
Moreover, we have the following characterizations of the adversarial risk:
\begin{subequations}
\begin{align}
  \mathcal{C}_{\mathsf{adv}}(\learnedfunc,2\pi/h(\B,n)) &\leq \frac{\sqrt{2} b \cdot h(\B,n)}{\pi h(\B,n) (2a - \sqrt{2}b)} + \frac{n}{2h(\B,n)}\label{eq:advsmall_abdn} \\
\mathcal{C}_{\mathsf{adv}}(\learnedfunc,1/n) &= \begin{cases}
1 \text{ if Equation~\eqref{eq:kstarlowerbound} holds, }\\
0 \text{ otherwise. }
\end{cases}\label{eq:advmed_abdn}
\end{align}
\end{subequations}
\end{theorem}

Theorem~\ref{thm:riskabdn} is a direct consequence of the characterizations of $k^*$ in Lemma~\ref{lem:k_star}, and is proved in Appendix~\ref{sec:riskabdnproof}.
Note that the expressions for classification risk, and adversarial risk to perturbations of $\mathcal{O}(1/\B)$ are \textit{upper bounds}; however, we will see that the implications for the bilevel ensemble are sharp in the asymptotic sense.
The expression for adversarial risk to perturbations of $\mathcal{O}(1/n)$ is actually exact.


\begin{table*}[t]
\caption{Asymptotic generalization errors for various values of $q$ (which
  controls how strongly min-norm learning favors the true feature by scaling it
  as $\sqrt{n^{p-q}}$ during the lifting) across regression, classification and
  adversarial robustness tasks. ``IF'' denotes
  independent-sub-Gaussian feature models studied earlier and ``LFF''
  denotes the lifted Fourier features of this paper.}
\label{tab:results}
\vskip 0.15in
\begin{center}
\begin{small}
\begin{sc}
\begin{tabular}{lcccc}
\toprule
Regime & Regression & \multicolumn{2}{c}{Classification}
  & Adversarial \\
 & LFF and IF & LFF &IF
  &(LFF, $\epsilon=1/n$) \\
\midrule
$q < 1$ & 0 & 0 & 0  & 0 \\
$1 < q < 1 + \frac{p-1}{2}$ & {\color{red}1}   & 0  & 0   & {\color{red}1}    \\
$1 + \frac{p-1}{2} < q < p$ & {\color{red}1}   & 0  & {\color{red}$\frac{1}{2}$}   & {\color{red}1} \\
$q > p$ & {\color{red}1}   & {\color{red}{$\frac{1}{2}$}}  & {\color{red}$\frac{1}{2}$}   & {\color{red}1}
\end{tabular}
\end{sc}
\end{small}
\end{center}
\vskip -0.1in
\end{table*}


We now state the direct consequences of Theorem~\ref{thm:riskabdn} for the classification and adversarial risk for the minimum-$\ell_2$-norm interpolator in the bilevel ensemble.
\begin{corollary}\label{cor:riskbilevel}
The minimum-$\ell_2$-norm interpolator of constant-labeled, regularly spaced data using Fourier features with the bilevel ensemble (parameterized by $p$ and $q$) incurs classification risk
\begin{align}\label{eq:risk_bilevel}
\mathcal{C}(\learnedfunc) \leq  \min\{ 1/2, C \cdot n^{q - p}\},
\end{align}
and adversarial risk
\begin{subequations}
\begin{align}
\mathcal{C}_{\mathsf{adv}}(\learnedfunc, 2\pi/h(\B,n)) &\leq \min\{1, C' \cdot n^{q - p} \} \label{eq:advsmall_bilevel}\\
\mathcal{C}_{\mathsf{adv}}(\learnedfunc,2/n) &= \begin{cases}
1 \text{ if } q > 1 \text{ and } n \geq n_0(q) ,\\
0 \text{ if } q < 1 .
\end{cases} \label{eq:advmed_bilevel}
\end{align}
\end{subequations}
where $C,C'$ are universal positive constants that do not depend on $n$, and $n_0(q)$ is a strictly decreasing function in $q$.
The function $n_0(q)$ is explicitly specified in the appendix.
\end{corollary}

Corollary~\ref{cor:riskbilevel} uncovers the existence of separating regimes between regression, classification, and adversarial robustness in the bilevel ensemble.
See Table~\ref{tab:results} for a comparison between the three tasks, and a comparison to equivalent evaluations under independent-feature models.
We elaborate on two special cases below:
\begin{enumerate}
\item The case $q \in (0,1)$, for which we have $\lim_{n \to \infty} \mathcal{C}(\learnedfunc), \mathcal{C}_{\mathsf{adv}}(\learnedfunc, 2\pi/h(\B,n)), \mathcal{C}_{\mathsf{adv}}(\learnedfunc,2/n) = 0$.
This corresponds to the classical regime studied for kernel ridge regression in which signal is preserved.
Clearly, the preservation of signal precludes our mechanism for the creation of adversarial examples\footnote{This would yield a corresponding illustration to Figure~\ref{fig:illustrationresult} in which the constant-signal component is close to $1$ instead of $0$!}.
\item The case $q \in (1,p)$, for which we have $\lim_{n \to \infty} \mathcal{C}(\learnedfunc), \mathcal{C}_{\mathsf{adv}}(\learnedfunc, 2\pi/h(\B,n)) = 0$ but $\lim_{n \to \infty} \mathcal{C}_{\mathsf{adv}}(\learnedfunc,2/n) = 1$.
Corollary~\ref{cor:riskbilevel} shows that in this regime, classification generalizes well despite the failure of regression.
However, the minimum-$\ell_2$-norm interpolator is highly sensitive to perturbations on the order of $1/n$; thus, \emph{there is a stark separation in classification and adversarial risk in this regime, even though the magnitude of the perturbations also goes to $0$ as $n \to \infty$.}
\end{enumerate}

Remarkably, the second regime includes both the cases $1 < q < (p+1)/2$ and
$(p+1)/2 < q < p$: the former corresponds to good generalization of
binary classification tasks in ensembles of sub-Gaussian independent
features, but the latter does
not~\cite{muthukumar2020classification}. Hence, classification on the Fourier-features of this paper is less sensitive to
having a strong weight on the true feature than that on independent-features. The reason is that, in the present case, the ``contamination"
of the spurious features is concentrated around the training points due to Gibbs phenomenon and is low everywhere else;
hence, even though the overall level of
``contamination" may be high, the contribution of the spurious features to prediction on a randomly drawn test point is low. This demonstrates
a qualitative difference between linear-model-on-features and linear-model-on-data, and motivates the study of feature-maps $\phi$ with
different properties. 

Corollary~\ref{cor:riskbilevel} highlights a phase transition at $q = 1$: adversarial examples will
not occur for $q < 1$, and adversarial examples will occur for $q > 1$.
Second, for a given $q > 1$, there is a phase transition at $n_0(q)$: adversarial examples will always occur for a large enough number of training samples.
We will now empirically demonstrate these phase transitions for models with both regularly spaced and uniform-at-random data.

\section{Random-Fourier-Sum (RFS) Model}
\label{sec:rfs}

The Fourier features that we have considered thus far provide a precise proof-of-concept of the idea of ``generalizable but brittle" overparameterized models.
In this section, we connect the stylized Fourier featurization to examples of ``weak-feature" families proposed in the literature that admit an explicit benefit in overparameterization.
We do this by introducing a family of \emph{random-Fourier-sum} (RFS) models that exhibit the well-known \emph{double descent} behavior: in particular, the classification test error decays with the number of random features introduced in the model.
We show that the behavior of both the classification and adversarial test error in the RFS model approaches that of the original stylized Fourier model in a scaling limit where the number of RFS features grows at a sufficiently fast rate relative to the number of training examples.
In Section~\ref{sec:simulations} , we explore an alternate slower growth rate of the number of RFS features for which the behavior is instead analogous to the more commonly studied \emph{independent-feature} models, i.e. where each component of the feature vector is independent. For example, most theoretical analyses of double descent assume independent-feature models, with the notable exception~\cite{mei2019generalization}.

We formally define the RFS model below.
Each RFS feature is expressed as a linear combination of $B$ Fourier features, where the weights in the linear combination are sampled from a Gaussian distribution\footnote{The analysis that we provide can also be readily extended to independent sub-Gaussian weights.}. 
The Fourier features and labels are given by $y = 1$, just as in Section~\ref{sec:setup}, and $B$ is a bandwidth parameter that controls how many Fourier features are represented in each RFS feature.
Finally, $\gamma$ parameterizes the strength of the true Fourier feature in each RFS feature, similarly to the bilevel weighting in Definition~\ref{def:bilevel_covariance}.
This featurization can be visualized as a neural network with two linear layers that uses the $B$ Fourier features as inputs to the second layer; see Figure~\ref{fig:rfs_nn} for this visualization.
The first layer is trained ``lazily'' (using the nomenclature of~\cite{chizat2019lazy}) by sampling from the Gaussian distribution and the parameters in the second layer are chosen by the minimum-$\ell_2$-norm interpolator from Definition~\ref{def:minl2}.
Of course, this ``lazy-training'' outcome is very different from the actual training outcome of a neural network.
However, it preserves one of the central difficulties in analyzing adversarial robustness, i.e. the non-linear effect of an adversarial perturbation on input data.

\begin{definition}[Random Fourier Sum Features $(n, p, q, T)$]
  \label{def:rfs}
  Let $\Sigmabold(n,p,q)$ (henceforth abbreviated to $\Sigmabold$) be a diagonal matrix as specified above in Definition~\ref{def:bilevel_covariance}, and set the number of features $d = n^T$. Then each entry of $\phirfs(x) \in \R^d$ is defined as $\forall i \in [d], \phirfs_i (x) = \langle \boldsymbol{z_i}, \phifour(x) \rangle$. 
  Here, $\boldsymbol{z_i} \text{ i.i.d.} \sim \mathcal{N}(\boldsymbol{0}, \Sigmabold)$ for all $i \in [d]$.
\end{definition}
Similar to Definition~\ref{def:minl2}, we define the minimum-$\ell_2$-norm interpolator on RFS features as
\begin{align}
 \label{eq:beta_2}
 \widehat{\betabold}_2 & := \argmin_{\betabold \in \re^d} \|\betabold\|_2 \\
	  & s.t \; \; \phirfs(\xtrain) \betabold = \ytrain \nonumber .
\end{align}
We define the learned function induced by these coefficients as $\widehat{g}(x) = \langle \widehat{\betabold}_2, \phirfs(x) \rangle$.

Figure~\ref{fig:rfs_dd} plots the test classification and regression loss of $\widehat{g}$ as a function of the number of RFS features, and clearly shows the double-descent behavior~\cite{belkin2019reconciling,geiger2019jamming}.
To better understand why this occurs, let us now connect at a high level the RFS model to other random-feature models that were recently shown to display the same double-descent behavior.
In particular, we consider the mathematical smodels provided by~\cite{belkin2019two,mei2019generalization,rahimi2008random}.
While there are important semantic differences in the modeling assumptions in these papers, in all cases the features are random and individually too ``weak" to ensure accurate prediction by themselves.
Hence, generalization error reduces when the number of features are increased due to both an
approximation-theoretic benefit and a reduction in ``variance'' from fitting noise.
This effect is reminiscent of the benefits of overparameterization in ensemble models like random forests and boosting, which were recently discussed in an insightful unified manner~\cite{wyner2017explaining}.
The RFS features serve as a archetypal example of such ``weak-features" since each RFS feature contains some component of the true constant feature, but is diluted by contributions from the spurious higher-frequency Fourier features as well.

Our RFS features are qualitatively similar to the random-Fourier-features (RFF) of~\cite{rahimi2008random}, that are shown in \cite{belkin2019two} to exhibit double-descent. However, in the RFF, the input-data is high-dimensional and the frequencies of the Fourier-features are random. In our case, the input is low-dimensional in order to constrain the adversary's power, and each RFS feature is obtained as a random linear-combination of Fourier features with fixed frequencies. 
The RFS features are also distinct from the features considered by \cite{mei2019generalization} in a similar manner.

\tikzset{%
every neuron/.style={
circle,
draw,
minimum size=1cm
},
neuron missing/.style={
draw=none,
scale=3,
text height=0.333cm,
execute at begin node=\color{black}$\vdots$
},
}

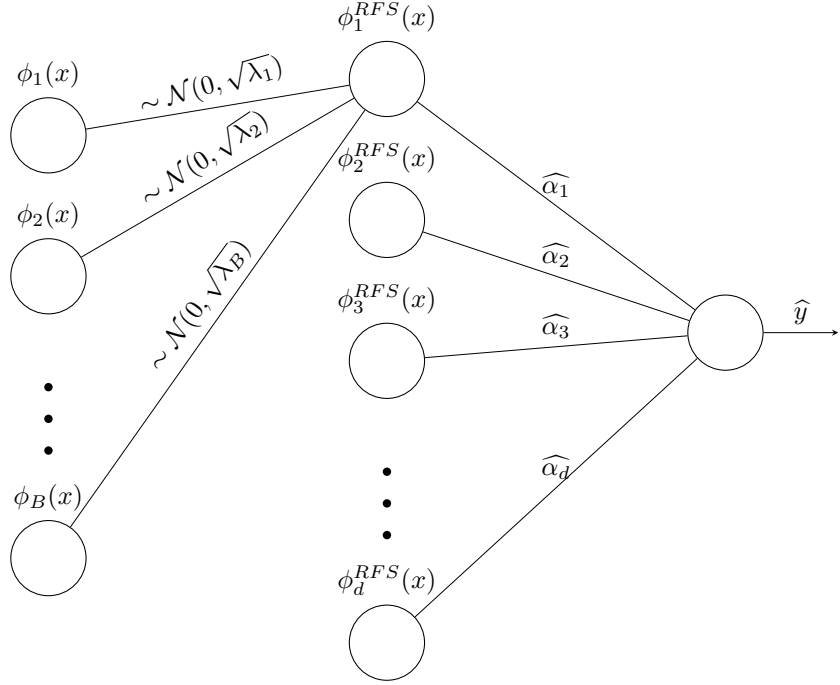
\begin{figure}
  \begin{center}
    \begin{tikzpicture}[x=1.5cm, y=1.5cm, >=stealth]

    \foreach \m/\l [count=\y] in {1,2,missing,3}
      \node [every neuron/.try, neuron \m/.try] (input-\m) at (4,2-\y*1.25) {};
    \foreach \m [count=\y] in {1,2,3,missing,4}
    \node [every neuron/.try, neuron \m/.try ] (hidden-\m) at (7, 2.5 - \y*1.25){};

    \foreach \m [count=\y] in {1}
      \node [every neuron/.try, neuron \m/.try ] (output-\m) at (10,-1) {};

    \foreach \l [count=\i] in {1,2,B}
      \node [above] at (input-\i.north) {$\phi_\l (x)$};

    \foreach \l [count=\i] in {1,2,3,d}
    \node [above] at (hidden-\i.north) {$\phi^{RFS}_\l (x)$};

    \foreach \l [count=\i] in {1}
      \draw [ultra thin, ->] (output-\i) -- ++(1,0)
	node [above, midway] {$\widehat{y}$};

    \foreach \i in {1,...,2}
      \foreach \j in {1}
      \draw [ultra thin] (input-\i) edge node[above, rotate = \i*(15+1*\i)] {$\sim \mathcal{N}(0, \sqrt{\lambda_{\i}})$} (hidden-\j);

  \draw [ultra thin] (input-3) edge node[above, rotate = 3*(15+1*3)] {$\sim \mathcal{N}(0, \sqrt{\lambda_{B}})$} (hidden-1);

    \foreach \i in {1,...,3}
      \foreach \j in {1}
      \draw [ultra thin] (hidden-\i) edge node[above] {$\widehat{\alpha_{\i}}$}(output-\j);

      \draw [ultra thin] (hidden-4) edge node[above] {$\widehat{\alpha_{d}}$}(output-1);
    \end{tikzpicture}
  \end{center}
\caption{RFS features visualized as a ``lazily"-trained neural network
  with two linear layers after a nonlinear layer that generated
  Fourier features from the input data, denoted by $x$.}
\label{fig:rfs_nn}
\end{figure}

\begin{figure}[htpb]
  \begin{subfigure}[]{0.33\textwidth}
  \begin{center}
    \includegraphics[width=\textwidth]{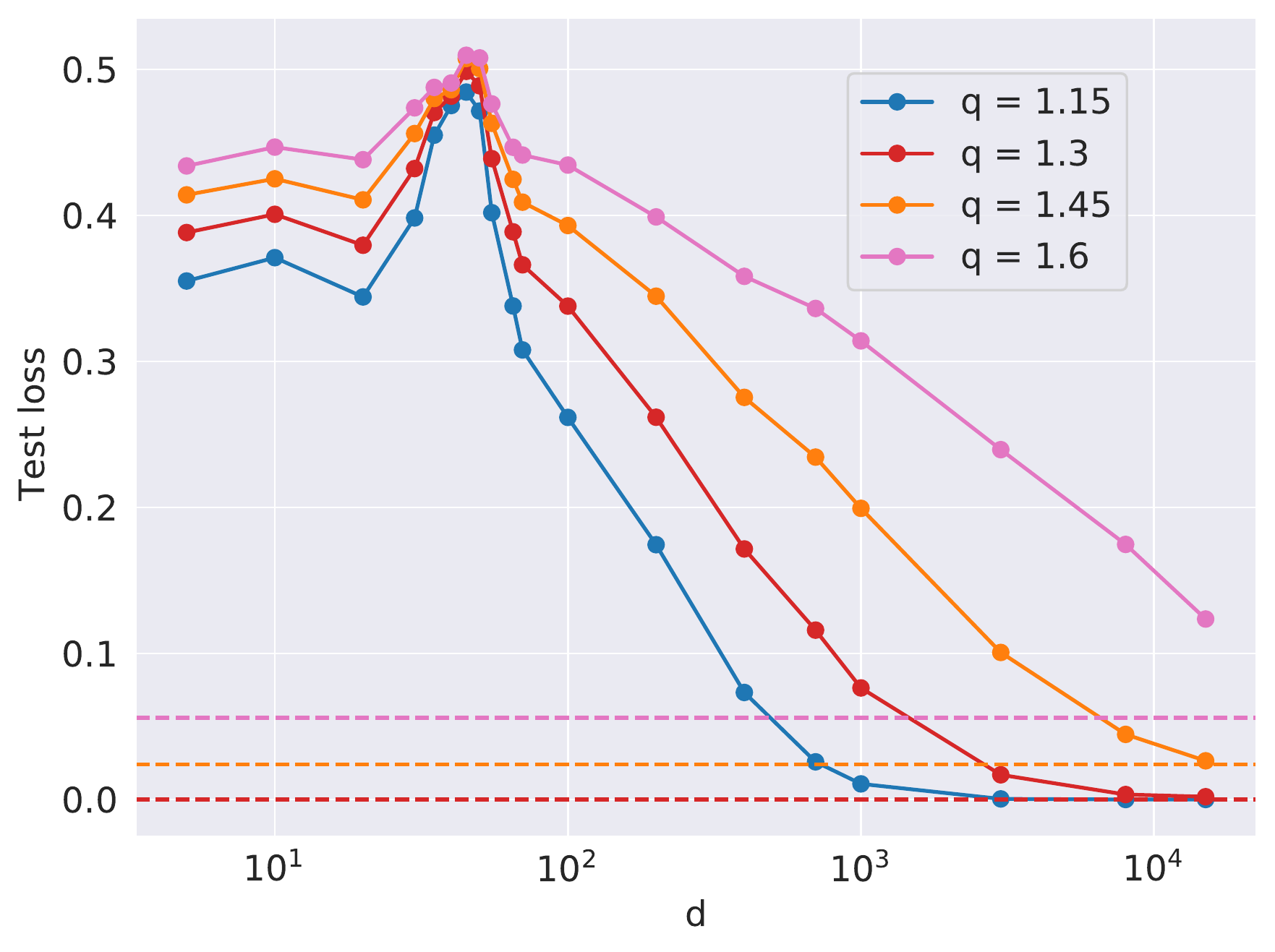}
  \end{center}
  \caption{Test classification loss}
  \label{subfig:rfs_dd_class}
  \end{subfigure}
  \begin{subfigure}[]{0.33\textwidth}
  \begin{center}
    \includegraphics[width=\textwidth]{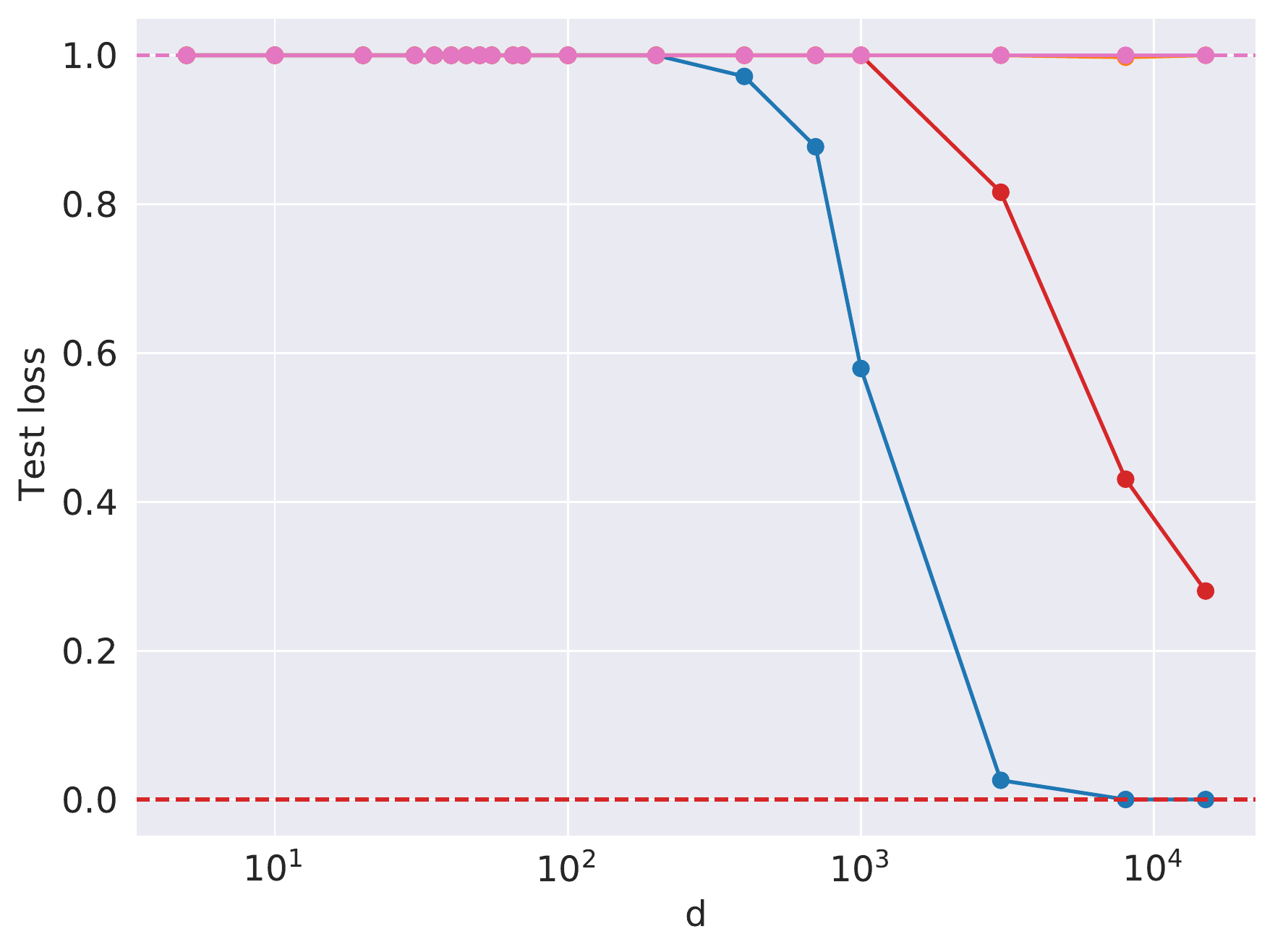}
  \end{center}
  \caption{Test adversarial loss}
  \label{subfig:rfs_dd_class_adv}
  \end{subfigure}
  \begin{subfigure}[]{0.33\textwidth}
  \begin{center}
    \includegraphics[width=\textwidth]{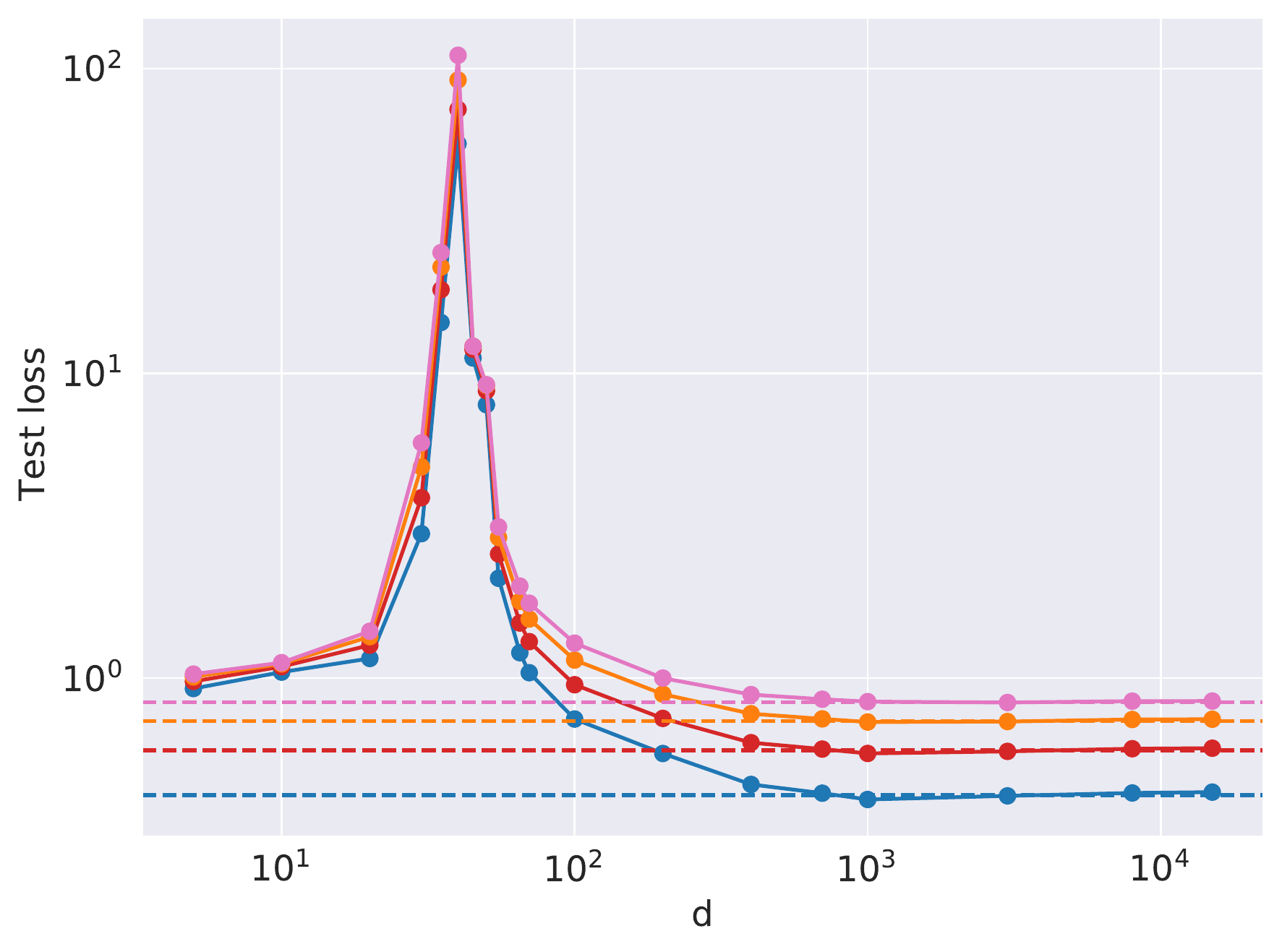}
  \end{center}
  \caption{Test regression loss}
  \label{subfig:rfs_dd_reg}
  \end{subfigure}
  \caption{Random Fourier Sum features benefit from overparameterization for three types of test error: regression, classification, and adversarial error (measured to perturbations of magnitude $1/n$). All the parameters scale as specified in Definition~\ref{def:rfs}; we fix $n=40, p = 2$, and plot different values of $q$ as specified in the legend of Figure~\ref{subfig:rfs_dd_class}.
  The dashed line indicates the test error incurred by the minimum-$\ell_2$-norm interpolator on the original Fourier features, which also appears to be the performance approached by the RFS features as $d \to \infty$.
  Theorem~\ref{thm:rfs_large_d} confirms this behavior.
  }
  \label{fig:rfs_dd}
\end{figure}

Since each RFS feature only places some weight on the true Fourier feature,
it satisfies the criteria for weak-features laid out by~\cite{belkin2019two}.
This is precisely why the RFS model exhibits double-descent and benefits from overparameterization.
A close look at Figure~\ref{fig:rfs_dd} reveals that as the number $d$ of RFS features is increased, the generalization error in all cases (regression, classification and adversarial test error) approaches the performance that would be obtained if instead of having RFS features, we just had clean access to the $B$
sinusoidal features weighted as per the bi-level model studied in this paper, i.e. the test error that was characterized in Theorem~\ref{thm:riskabdn}.
The following theorem shows that this does, in fact, happen provided that the number of RFS features (given by $d$) grows at a sufficiently fast rate relative to the number of underlying Fourier frequencies (given by $B$) and the number of training examples (given by $n$).

\begin{theorem}[RFS converges to Fourier]
  \label{thm:rfs_large_d}
  Let the training data be chosen at regularly-spaced intervals, and $\widehat{f}$ denote the interpolator of Fourier-features (Definition~\ref{def:minl2})
  with weighting $\Sigmabold(n, p, q)$ as specified in the Definition~\ref{def:bilevel_covariance};
  Further, let $\widehat{g}$ denote the interpolator of RFS features in Equation~\eqref{eq:beta_2}, where the features in Definition~\ref{def:rfs} are configured by $(n, p, q, T)$.
  Then, for $q > 1$, we have the following results:
\begin{enumerate}
  \item  If $T > 6p - 1$:
\[
  \lim_{n \to \infty} \classloss(\learnedfunc) = \lim_{n \to \infty} \classloss(\widehat{g})
 \]
 \item  If $T > 6p + 1 - 2q$:
\[
  \lim_{n \to \infty} \classlossadv \left(\learnedfunc, \frac{1}{n} \right) = \lim_{n \to \infty} \classlossadv \left(\widehat{g}, \frac{1}{n} \right)
 \]
\end{enumerate}
\end{theorem}

 \begin{figure}[htpb]
  \centering
  \includegraphics[width=0.8\textwidth]{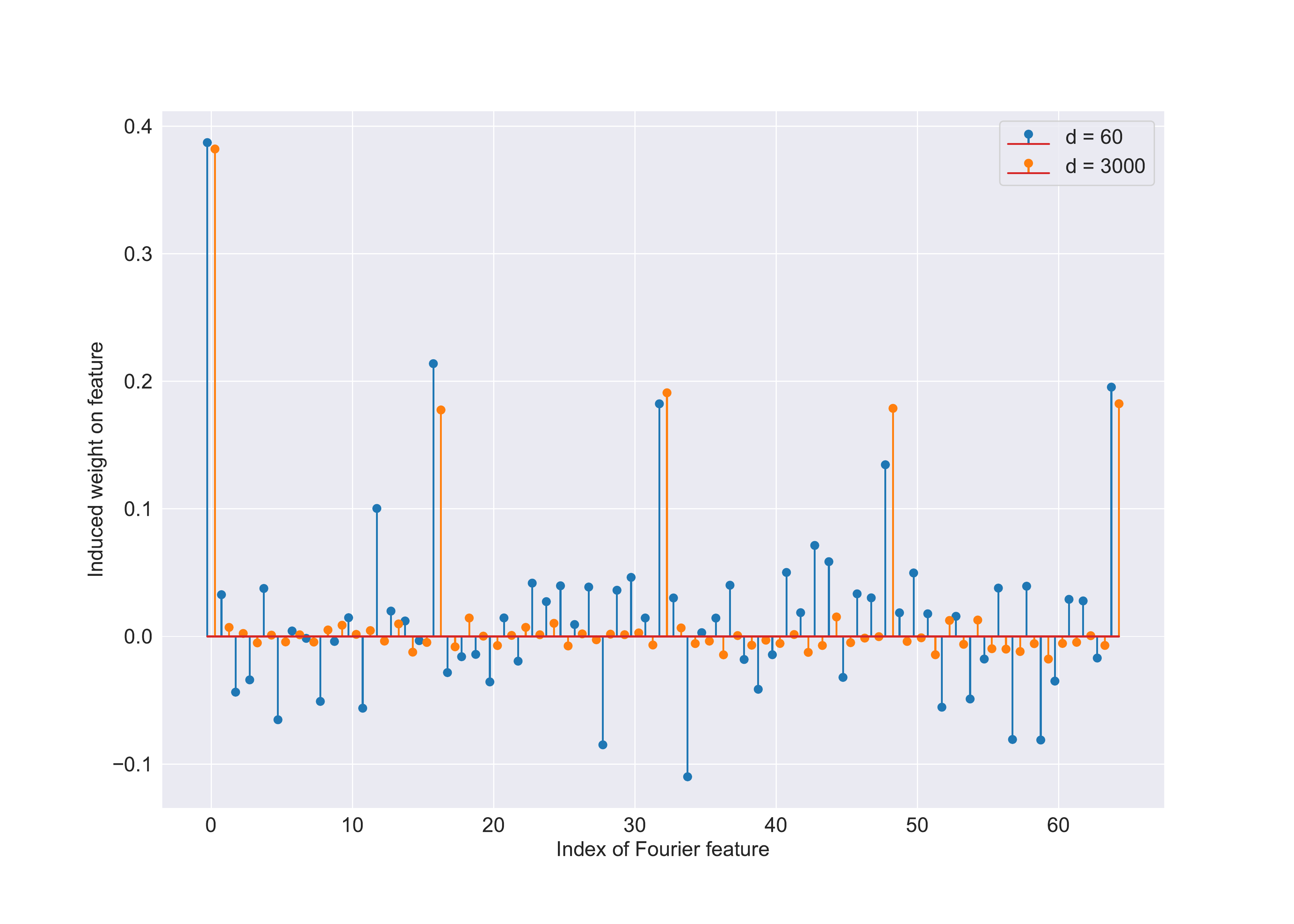}
  \caption{The effective learned coefficients on the $B=65$ underlying Fourier
    features, denoted by $\alphaRFS$, as the number $d$ of RFS features is increased. Here, $n = 8, p = 2, B=65$ and $q = 1.45$.}
  \label{fig:coeff_stem}
\end{figure}
To interpret the result of Theorem~\ref{thm:rfs_large_d}, it is useful to take a step back and interpret the RFS model in a slightly different way.
In particular, each of the RFS features can also be thought of as a noisy version of the
true (constant) feature since it is corrupted by contributions by the
spurious high-frequency Fourier terms.
Viewed from this lens, the each RFS feature is ``useful but non-robust" in the taxonomy of~\cite{ilyas2019adversarial};
Thus, Theorem~\ref{thm:rfs_large_d} is a concrete theoretical illustration of their explanation for adversarial examples.

Theorem~\ref{thm:rfs_large_d} shows that if the width $d$ of the second-layer of the neural-net in Figure~\ref{fig:rfs_nn} tends to infinity at a high enough rate (relative to $n$ and  $B$), the classification and adversarial risk of the following two interpolators are exactly equal:
\begin{enumerate}
  \item min-$\ell_2$ interpolator $\learnedfunc(\cdot)$ trained on the Fourier features output by the first layer.
  \item min-$\ell_2$ interpolator $\widehat{g}(\cdot)$ trained on the RFS features output by the second layer
\end{enumerate}
But $\learnedfunc(\cdot)$ above is exactly the interpolator that was shown in Section~\ref{sec:main_result} to exhibit generalizable-but-brittle behavior.

To understand the main strategy for the proof of Theorem~\ref{thm:rfs_large_d},
we first make the simple observation that any linear model on RFS features
is also a linear model on Fourier features, which follows from Definition~\ref{def:rfs},
and is expanded below.
\begin{equation}
  \label{eq:rfs_to_fourier}
  q(x) = \sum_{i=1}^d c_i \phirfs_i(x) = \sum_{i=1}^d c_i  \sum_{j=1}^B z_{ij} \phi_j(x) =  \sum_{j=1}^B \left(\sum_{i=1}^d c_i z_{ij} \right) \phi_j(x)
\end{equation}
Hence, we can introduce the effective coefficients, denoted by $\alphaRFS$, that the min-$\ell_2$-norm interpolator on the RFS features places on the original Fourier basis: 
It is easy to verify from Equation~\eqref{eq:rfs_to_fourier} that $\widehat{\alpha}^{\mathsf{RFS}}_i = \sum_{i=1}^d \widehat{\beta}_{2, i}  z_{ij}$.
Using results from random matrix theory, we show an explicit rate at which $\alphaRFS$ converge to $\coeffs_2$ in the $\ell_2$-norm metric.
This convergence can be understood by examining Figure~\ref{fig:coeff_stem} which plots the effective coefficients $\alphaRFS$ on the 64 underlying Fourier features that are learned for a case with $n=8$ and a particular favoring of the
constant feature. 
The two stem plots show what happens as we increase $d$, the number of RFS features; clearly, the effective coefficients on the constant feature and its appropriate cosine
aliases converge to non-zero values, while the effective
coefficients on the other underlying Fourier features are much smaller.
As $n \to \infty$, the magnitude of the effective coefficients on the non-alias features turns out to be vanishingly small relative to the magnitude of the effective coefficients on the constant feature and its aliases.

While this convergence of $\alphaRFS$ to $\coeffs_2$ directly implies convergence of the regression test error, proving the same result for classification and adversarial test error is more involved owing to the discontinuity of these loss functions.
Essentially, we use the non-asymptotic rate of convergence of $\alphaRFS$ to $\coeffs_2$ to show that the difference between the two functions $\learnedfunc(\cdot)$ and $\widehat{g}(\cdot)$ is bounded on the entire domain.
To obtain the asymptotic equivalence in error for $\learnedfunc(\cdot)$ and $\widehat{g}(\cdot)$, the steps for the proof for Theorem~\ref{thm:riskabdn}
are reproduced by ensuring that the bounded difference between the functions is not sufficient to flip the sign
of any specified test point. 
The formal argument is provided in Appendix~\ref{sec:rfsproof}.

\section{Empirical demonstrations}
\label{sec:simulations}

Theorems~\ref{thm:riskabdn} and~\ref{thm:rfs_large_d} make striking conceptual and quantitative predictions.
In particular, they predict for feature-familiies that exhibit spatial-localization
that (a) adversarial examples will emerge in the near vicinity of
training points when the fraction $a$ of the true signal component drops below a critical level; and (b) the ``contamination" effect from orthogonal components to the true signal can concentrate in a benign fashion in the near vicinity of training points, leading to surprisingly good generalization that is not admitted by corresponding independent-feature models.
In this section, we demonstrate that these predictions are empirically borne out well beyond the idealized Fourier feature map that we used for tractable mathematical analysis.

\subsection{Random Training Data}
The insights from our theory also allow us to predict the expected behavior of classification and adversarial test error on \emph{randomly drawn} training data, which is more common in machine learning.
Figure~\ref{fig:poppingplots} depicts the results of experiments
designed to validate the predictive power of insight (a),
which is that adversarial examples appear in the vicinity of training points when the coefficient $a$ on the true feature drops below a critical value  $a_c$.

The solid curves plot the coefficient $a$ on the true feature,
and the black curve is a natural theoretical prediction for $a_c$, detailed in the following paragraph.
The dashed curves plot the adversarial error resulting from an adversary capable of perturbing test points up to $\epsilon = \frac{1}{n}$.
In Figure~\ref{subfig:en_reg_four}, the performance of the regularly-spaced training case is illustrated, and as expected, the phase transition at which adversarial examples appear ubiquitously appears just as the corresponding solid curve drops below the black curve i.e when $a < a_c$.
Figure~\ref{subfig:en_random_four} illustrates
the performance with the same Fourier features and inductive bias, but
training points chosen uniformly at random.
We notice that the phase transition occurs more smoothly, but occurs at the same place.
It is also interesting to note the evolution of the adversarial error with the number of training points $n$.
In the regularly spaced case (Figure~\ref{subfig:en_reg_four}) the adversarial error is observed to approach $1$as $n \to \infty$, since every test point is within  $\frac{1}{n}$ distance of a training point.
However, in the randomly spaced case (Figure~\ref{subfig:en_random_four}), it is observed to approach $1-\frac{1}{e}$.
The intuition for this scaling is as follows: when $n$ training points are drawn uniformly at random, there is a probability of approximately $\frac{1}{e}$ that none of the training points will be within $\frac{1}{n}$ distance of an arbitrary test point.
This $\frac{1}{e}$ probability arises from a balls-and-bins-based calculation that is provided in Appendix~\ref{sec:randomlyspaced}.

\begin{figure}[htpb]
  \centering
  \begin{subfigure}[]{0.6\linewidth}
    \centering
    \includegraphics[width=\textwidth]{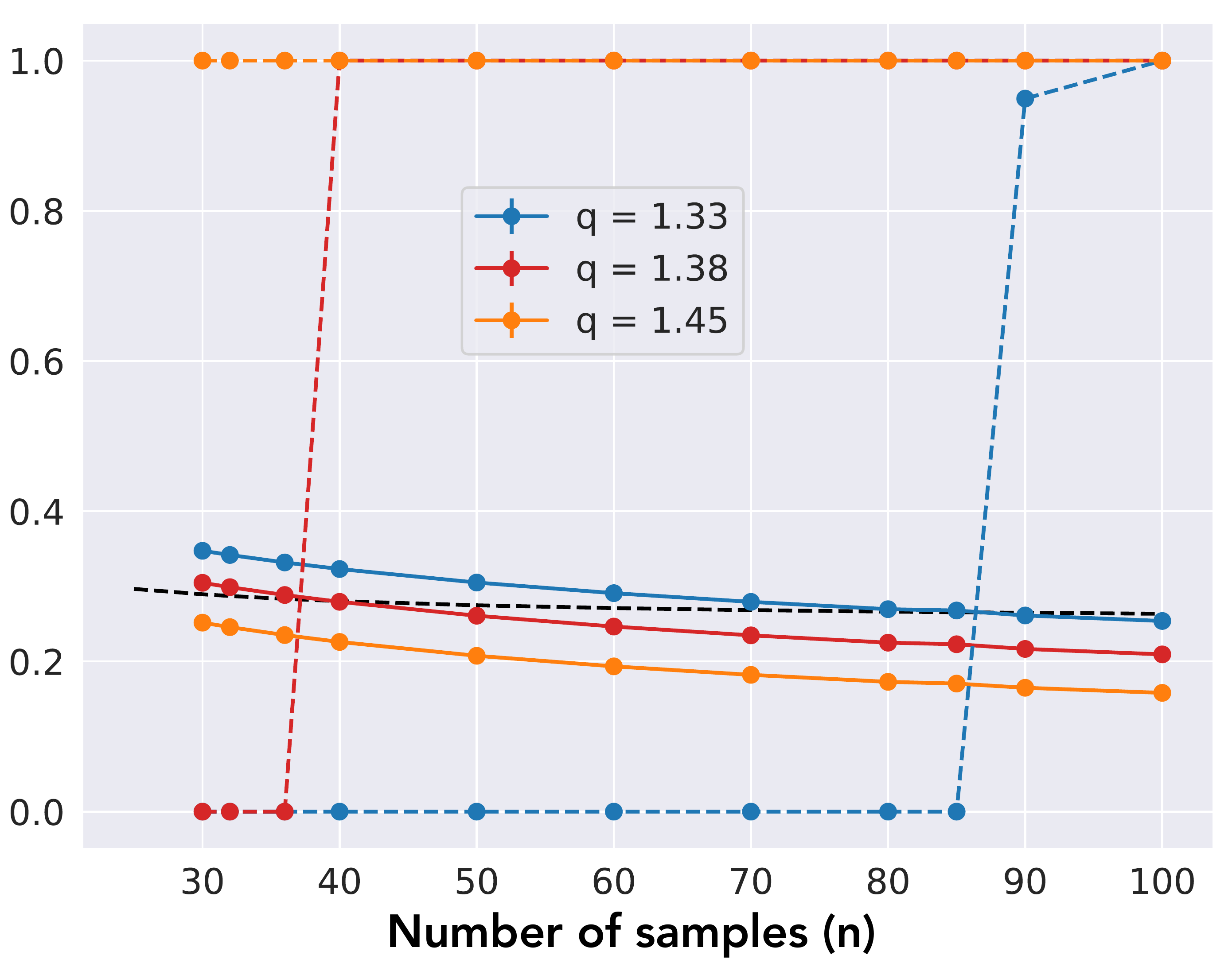}
  \caption{Fourier, regularly spaced}
  \label{subfig:en_reg_four}
  \end{subfigure}

  \begin{subfigure}[]{0.6\linewidth}
    \centering
    \includegraphics[width=\textwidth]{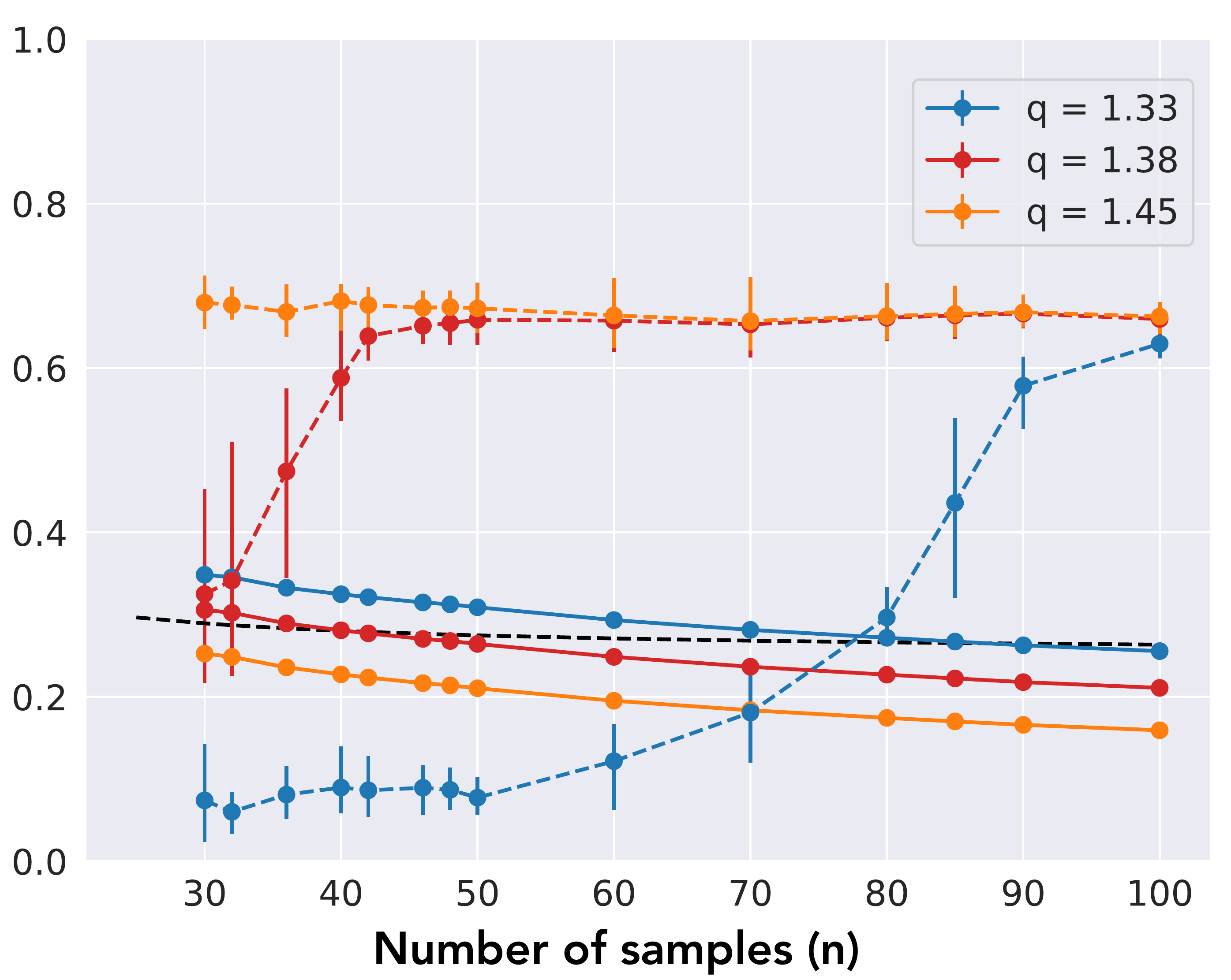}
  \caption{Fourier, uniformly sampled}
  \label{subfig:en_random_four}
  \end{subfigure}
  \caption{\textbf{Illustration of phase transition in Adversarial
      Error as a function of $n$:} All parameters vary according to
    the bilevel scalings in Definiton~\eqref{def:bilevel_covariance}
    with $p=2$ and the choice of $q$ for each curve as
    specified in the legend.
  }  \label{fig:poppingplots}
\end{figure}

To understand how the prediction for $a_c$ was obtained, first
recall from Equation~\ref{eq:form_of_learned_kernel} that the learned
function is expressed as a weighted sum of the constant function and the Dirichlet kernel as below:
\[
\widehat{f}(x) = \frac{2a - \sqrt{2} b}{2\sqrt{2}} + \frac{b}{2} \cdot D_{N_A}(n \pi x)
\]
The Dirichlet kernel has a characteristic sinc-like shape
with a global minimum of about $(N_A + 1) (-0.21)$. For any single period of $\widehat{f}(x)$, this is also the
global minimum of $\widehat{f}(x)$; call it $x^\ast$.
Hence, denoting $s = \frac{2a - \sqrt{2} b}{2\sqrt{2}}$, at $x^\ast$, we have
\[
  \widehat{f}(x^\ast) = s - \frac{b}{2} (N_A + 1) 0.21 = s + (1-s) 0.21
\]
If this evaluates to $< 0$, then there is a misclassification at $x^\ast$, and adversarial examples will
exist in the vicinity of each training point.
This condition is satisfied when $s$ is smaller than $\approx \frac{1}{6}$.
Meanwhile, to interpolate the training data, we know
that $\frac{a}{\sqrt{2}} + \frac{B}{2n}b = 1$ since $N_A =
\frac{B}{2n}$.
Solving these two equations yields a condition for $a$, which is denoted by the dashed black line.

\subsection{Legendre Polynomial Features}

\begin{figure}[htpb]
  \begin{center}
    \includegraphics[width=0.9\linewidth]{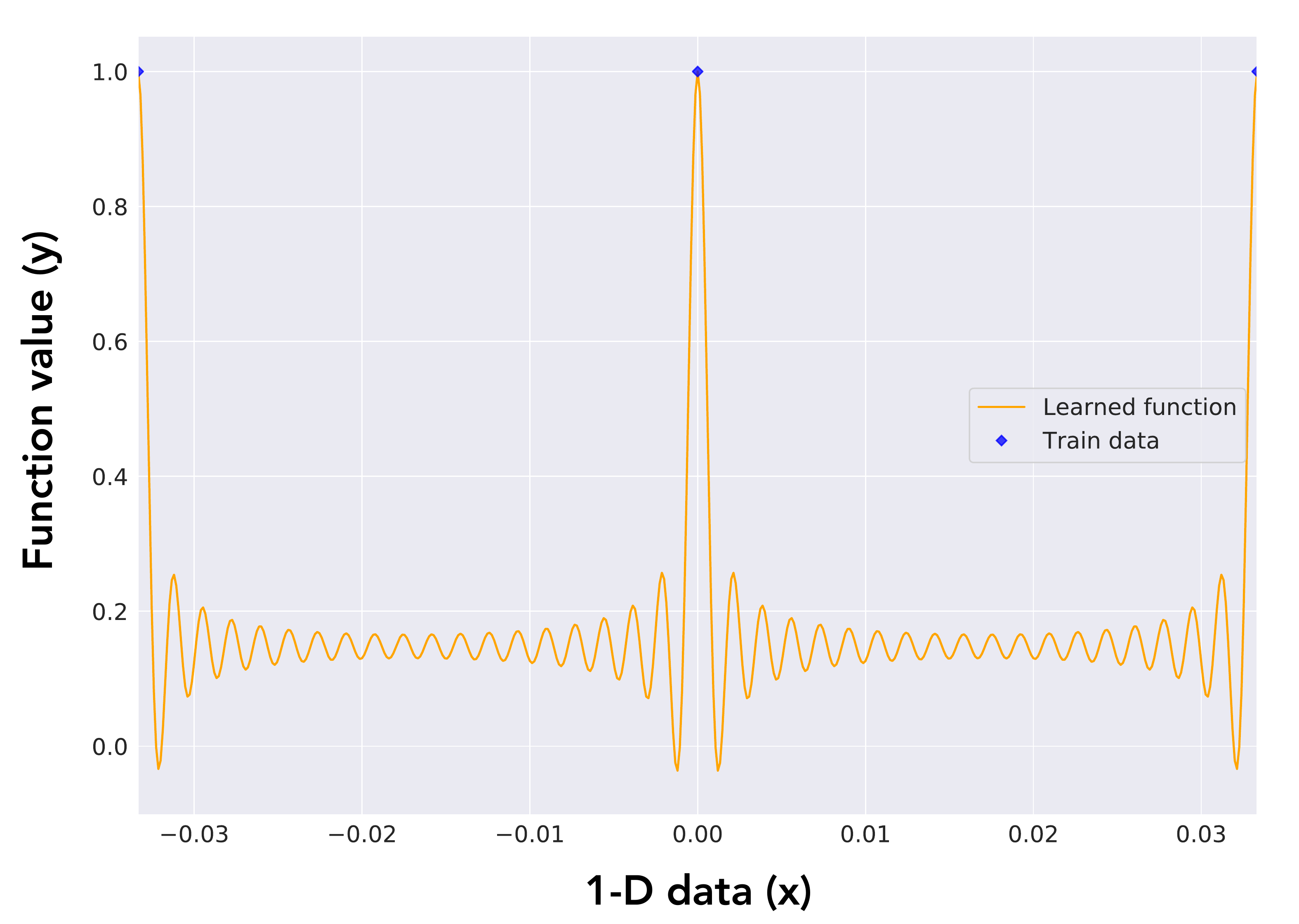}
  \end{center}
  \caption{Functions learned using Legendre features with
    regularly-spaced training data.}
  \label{fig:legendre_plot}
\end{figure}

The most basic question is whether the phenomena that we identified are restricted to Fourier features, or whether they are more general.
Accordingly, we test our predictions in the simple $1$-dimensional model, but with the Legendre polynomial feature map instead of the Fourier feature map.  Figure~\ref{fig:legendre_plot} illustrates the typical local behavior of functions learned using overparameterized Legendre polynomials for our ultra-toy example from Figure~\ref{fig:illustrationresult}; both the spatial localization and Gibbs-like overshoot is clearly visible.
In fact, there is a strong resemblance to the Dirichlet kernel from
Figure~\ref{fig:dirichletkernel}.

A similar phase transition behavior for adversarial examples as for the randomly spaced data above, is seen in the dashed curves in Figure~\ref{subfig:en_legendre} for Legendre polynomial features with regularly spaced training points.
The solid curves in Figure~\ref{subfig:en_legendre} also illustrate the continued applicability of insight (b) above, which is that spatial localization allows for unusually good generalization.
Notice that for each of the different values of $q$ (recall that $q$ controls the strength of the inductive bias towards recovering the true function component), the classification error decreases with increasing $n$.
The solid curves in all the subplots of Figure~\ref{fig:error_vs_n} experimentally confirm the ``asymptotically benign contamination'' behavior that the spatial localization bias engenders for both the Fourier and Legendre families.
As summarized in Table~\ref{tab:results}, for the larger values of $q$ an independent random feature model would instead have the classification error increasing with increasing $n$.

\begin{figure*}[htpb]
  \begin{subfigure}[]{0.33\linewidth}
    \centering
    \includegraphics[width=\textwidth]{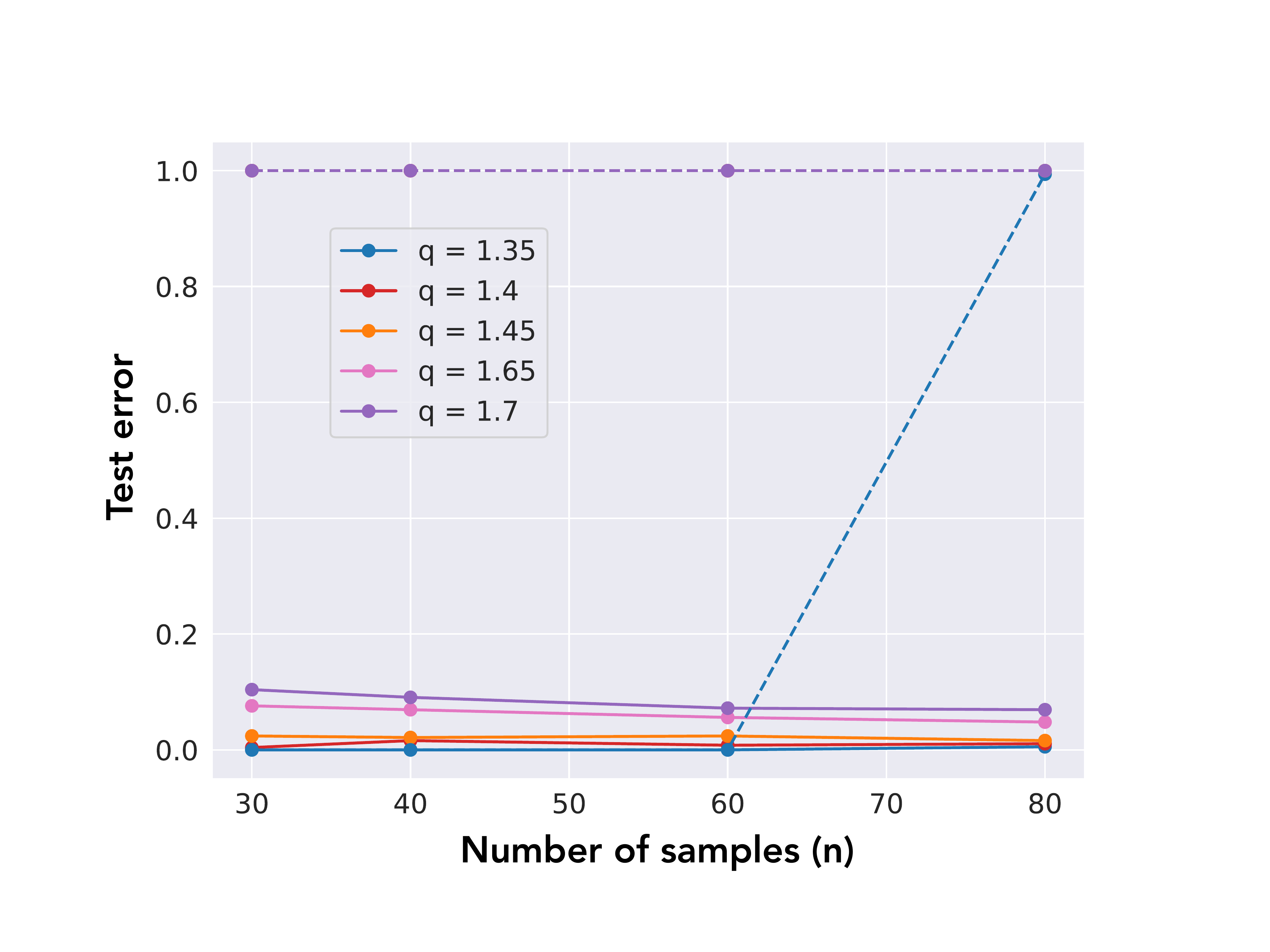}
  \caption{Fourier, regularly spaced}
  \label{subfig:en_reg_four}
  \end{subfigure}
  \begin{subfigure}[]{0.33\linewidth}
    \centering
    \includegraphics[width=\textwidth]{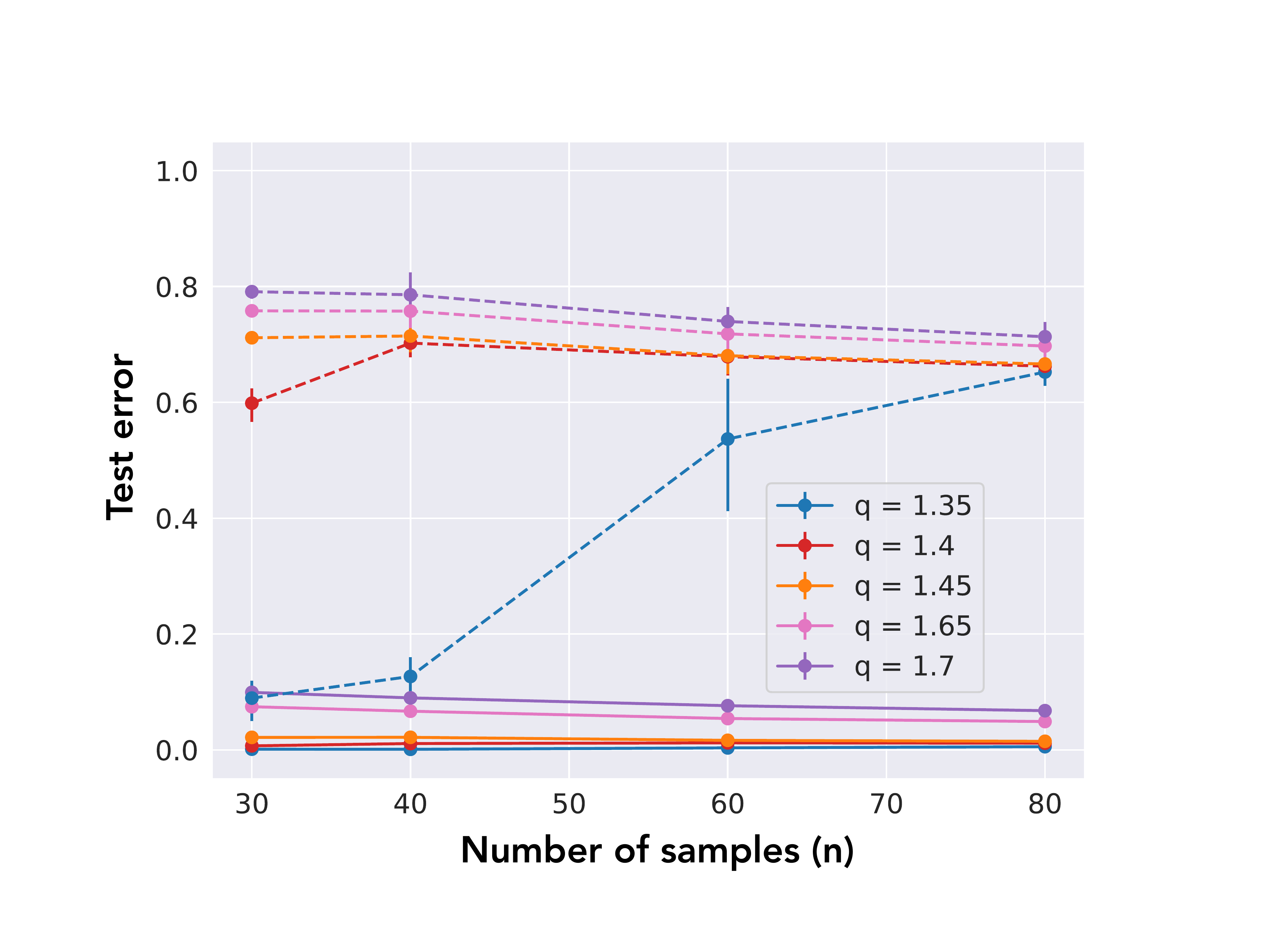}
  \caption{Fourier, uniformly sampled}
  \label{subfig:en_random_four}
  \end{subfigure}
  \begin{subfigure}[]{0.33\linewidth}
  \centering
  \includegraphics[width=\textwidth]{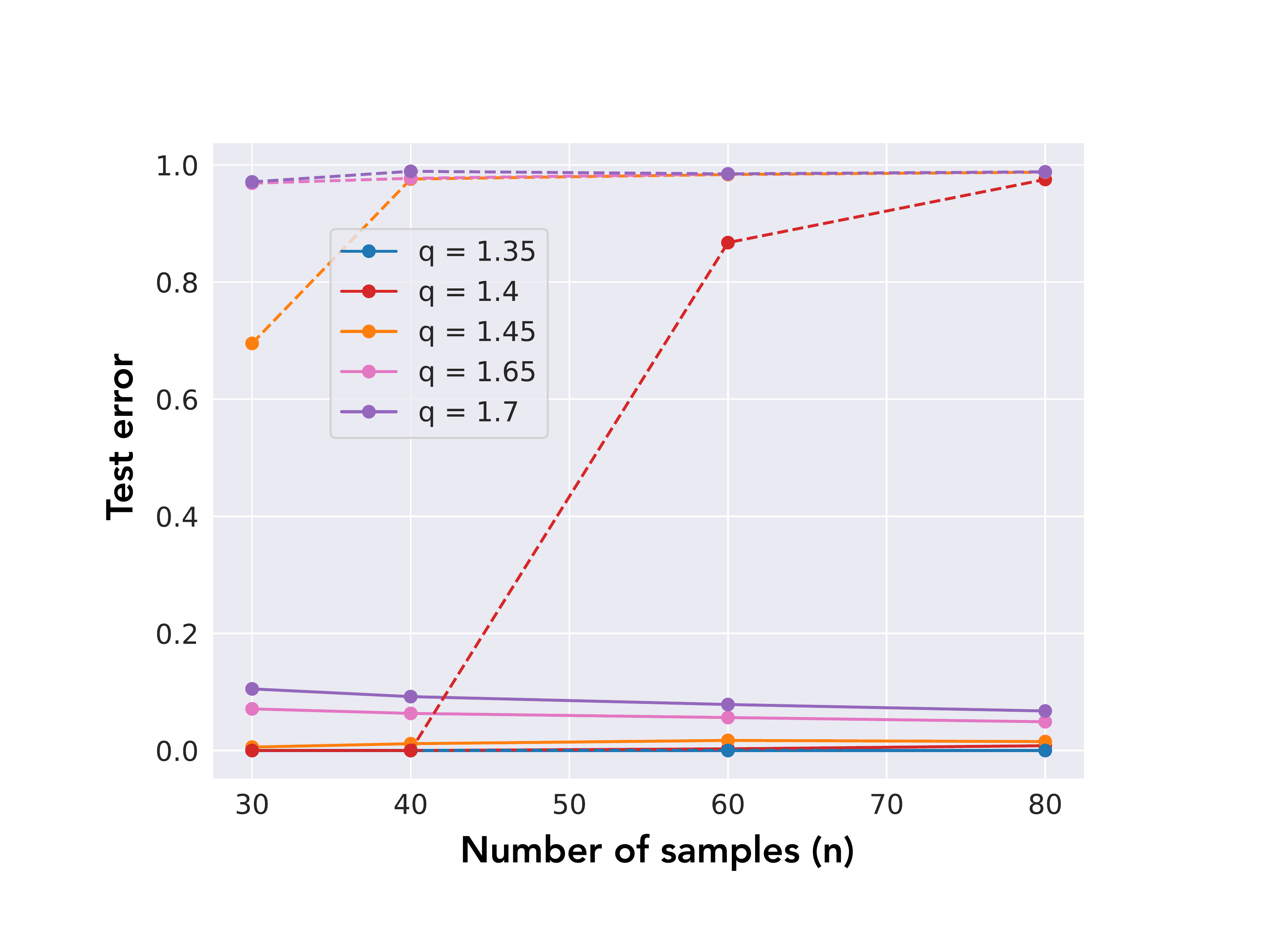}
  \caption{Legendre, regularly spaced }
  \label{subfig:en_legendre}
  \end{subfigure}
 \caption{\textbf{Classification (solid) and adversarial (dashed) risk as a function of
     n:} All parameters vary according to the bilevel scalings in
   Definiton~\eqref{def:bilevel_covariance} with $p=2$ and the
   choice of $q$ for each curve as specified in the legend.
 }
   \label{fig:error_vs_n}
\end{figure*}

\subsection{RFS Features: $d \ll B$}

In Section~\ref{sec:rfs}, we introduced the random Fourier sum (RFS) model.
We mathematically showed that if the number of RFS features, $d$, grows sufficiently faster than the number of original Fourier frequencies, $B$, the classification and adversarial test error arising from minimum-$\ell_2$-norm interpolation with the RFS features behave asymptotically exactly like the corresponding minimum-$\ell_2$-norm interpolator of the original Fourier features.
The latter essentially constitutes an \emph{infinite-width} limit of the RFS model.

Now, we explore the case where $B$ instead grows much faster than $d$.
The log-log plot shown in Figure~\ref{fig:rfs_converge_fourier} explores the effect of the number of RFS features $d$ on the amount of weight placed by the optimizer on aliases of the constant function ($\alphaRFSj, j = 2kn$) versus other non-alias features ($\alphaRFSj, j \neq 2kn$).
Recall that in the Fourier case, Lemma~\ref{lem:coeffs} illustrates that the weight placed on all non-aliases would be exactly $0$, and the weights placed on all aliases would be exactly equal to each other.

The prediction of Theorem~\ref{thm:rfs_large_d} is corroborated in this plot.
The solid lines denote the \emph{average recovered weight} of contaminating aliases $\frac{1}{N_A} \left( \sum_{k=1}^{N_A} \alphaRFSalias \right)$; these convere to the blacked dashed line ($b$ from Lemma~\ref{lem:coeffs}) that the original Fourier model would predict.
These contaminating features exhibit spatial localization and create Gibbs phenomena.
The plots also show the \emph{total energy} (squared norm) in the recovered coefficients on the non-alias Fourier features: \[
  \sum_{j = 2}^B \mathbb{I} \left[ \frac{j}{2n} \not \in \mathbb{Z} \right] (\alphaRFSj)^2
.\] 
This energy is clearly decaying to zero as the number of RFS features increases.
However, for small values of $d$, the energy in these features gives rise to non-localized contamination of the form that exists in the independent feature models of~\cite{muthukumar2020classification}.
Hence, it would stand to reason that adversarial examples in the $d \ll B$ case would not be as tightly concentrated around training points as the $d \gg B$ case.

\begin{figure}[htpb]
  \begin{subfigure}[]{0.5\textwidth}
  \begin{center}
   \includegraphics[width=\textwidth]{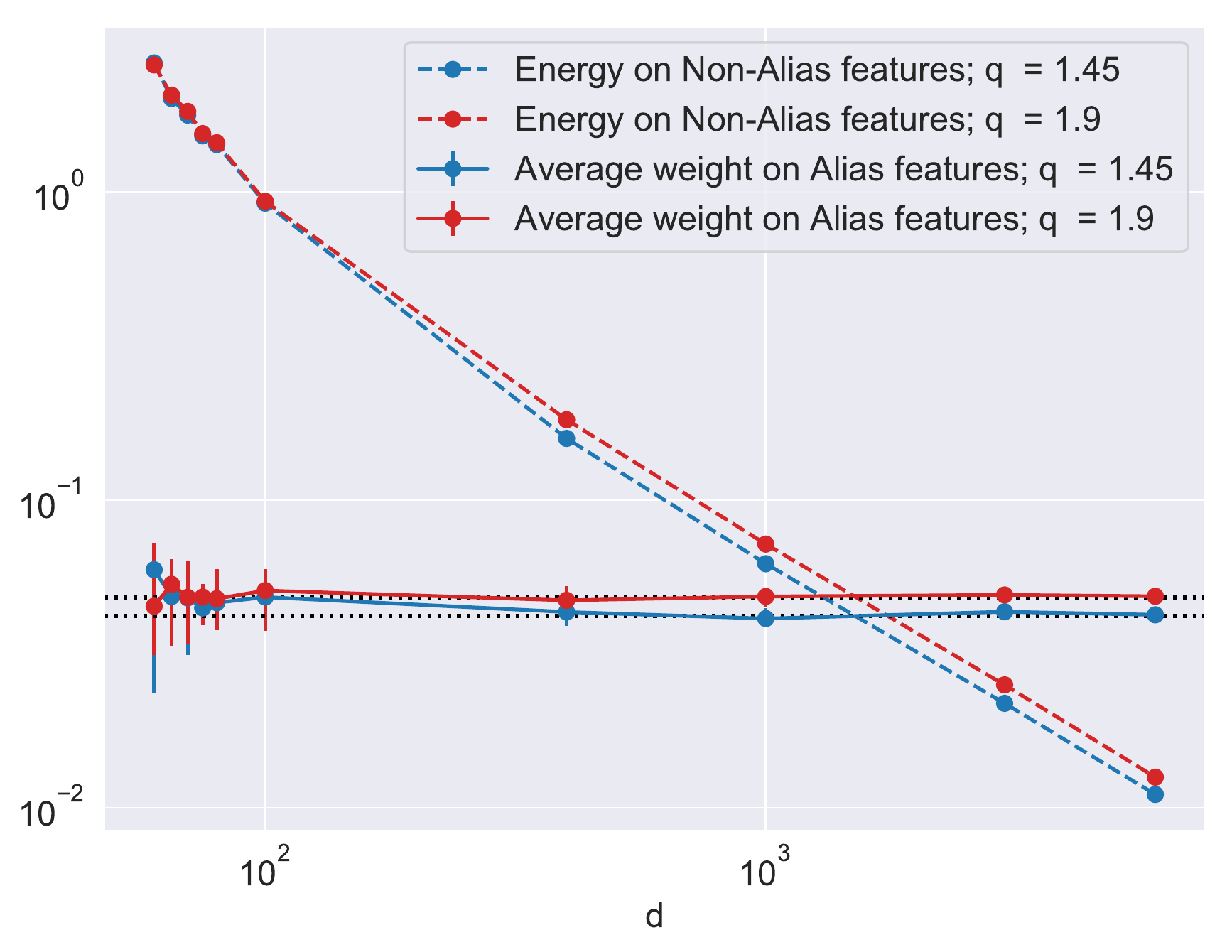}
  \end{center}
  \caption{$n=40, B = 1601$}
  \label{fig:n40}
  \end{subfigure}
  \begin{subfigure}[]{0.5\textwidth}
  \begin{center}
   \includegraphics[width=\textwidth]{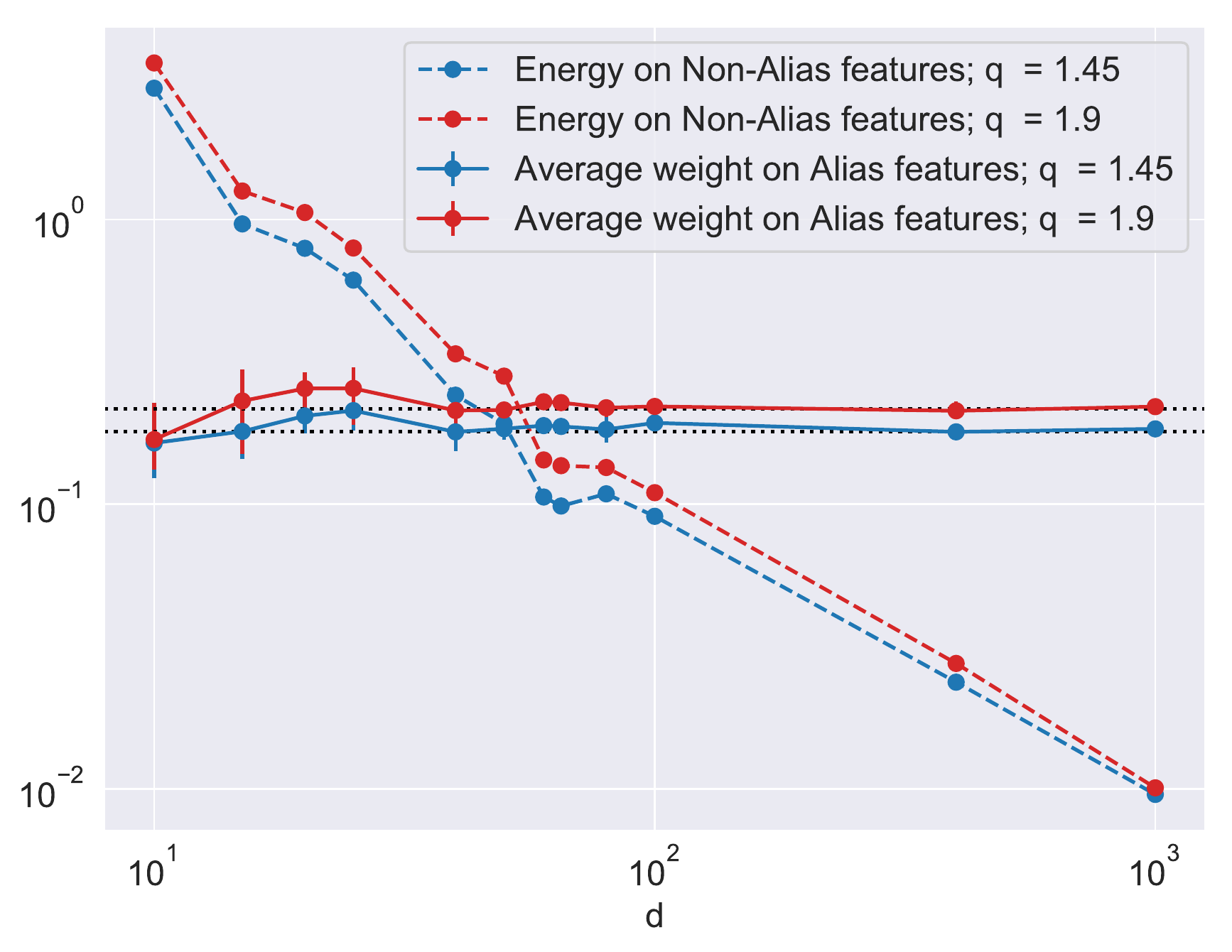}
  \end{center}
  \caption{$n=8, B = 65$}
  \label{fig:n8}
  \end{subfigure}

 \centering
 \caption{In the limit as $d \to \infty$, the min-L2 interpolator
   learned on RFS features converges to that learned on bi-level
   weighted Fourier features. The error bars on the alias weights
   capture the variation in weights on the different aliases
   ($30_{th}$ and $70_{th}$) percentile. Fix $p=2$ for both
   figures. The black dotted lines are what would be learned in the
   bi-level Fourier features model. See the accompanying text for precise definitions of ``energy'' and ``average weight''.
 }  \label{fig:rfs_converge_fourier}

\end{figure}

Indeed, this is what happens: let us look at the CDF of the distance to the nearest training point for misclassified test points.
Figure~\ref{fig:cdf} shows that the CDFs sharply concentrate close to the training points for the Fourier case, and that this behavior is very different from the uniform distribution that we would expect for independent features.
We can see that as the number of RFS features, $d$, increases, the CDFs start approaching the curve for the Fourier features case, showing that the Gibbs phenomenon is indeed encouraging misclassification near training points.
As long as this Gibbs phenomenon is active, there are adversarial examples to perturbations at the scale of $\frac{1}{n}$ as illustrated in Figure~\ref{fig:rfs_separation}.
However, for very small values of $d$, adversarial examples persist, but are now as likely to be found away from training points as in their vicinity.
This is illustrated by the solid red curve.

From these plots, it is clear that the RFS features model possesses sufficient flexibility to capture behaviors present in the original Fourier model and independent-feature models.
In particular, if $d \ll B$, the independent features model behavior persists; but if $d \gg B$, it fades away.
What about misclassifications?
In the Fourier model, they exist because of Gibbs phenomena near training points.
In the independent features model, they are a consequence of the independent contaminating features randomly overcoming the effect of (barely) surviving signal.
The RFS model is observed to flexibly model both effects.
It is an intriguing question for future work to sharply analyze the $d \ll B$ regime and mathematically show the replication of independent-feature behavior.

\begin{figure}[htpb]
  \centering
  \includegraphics[width=0.8\textwidth]{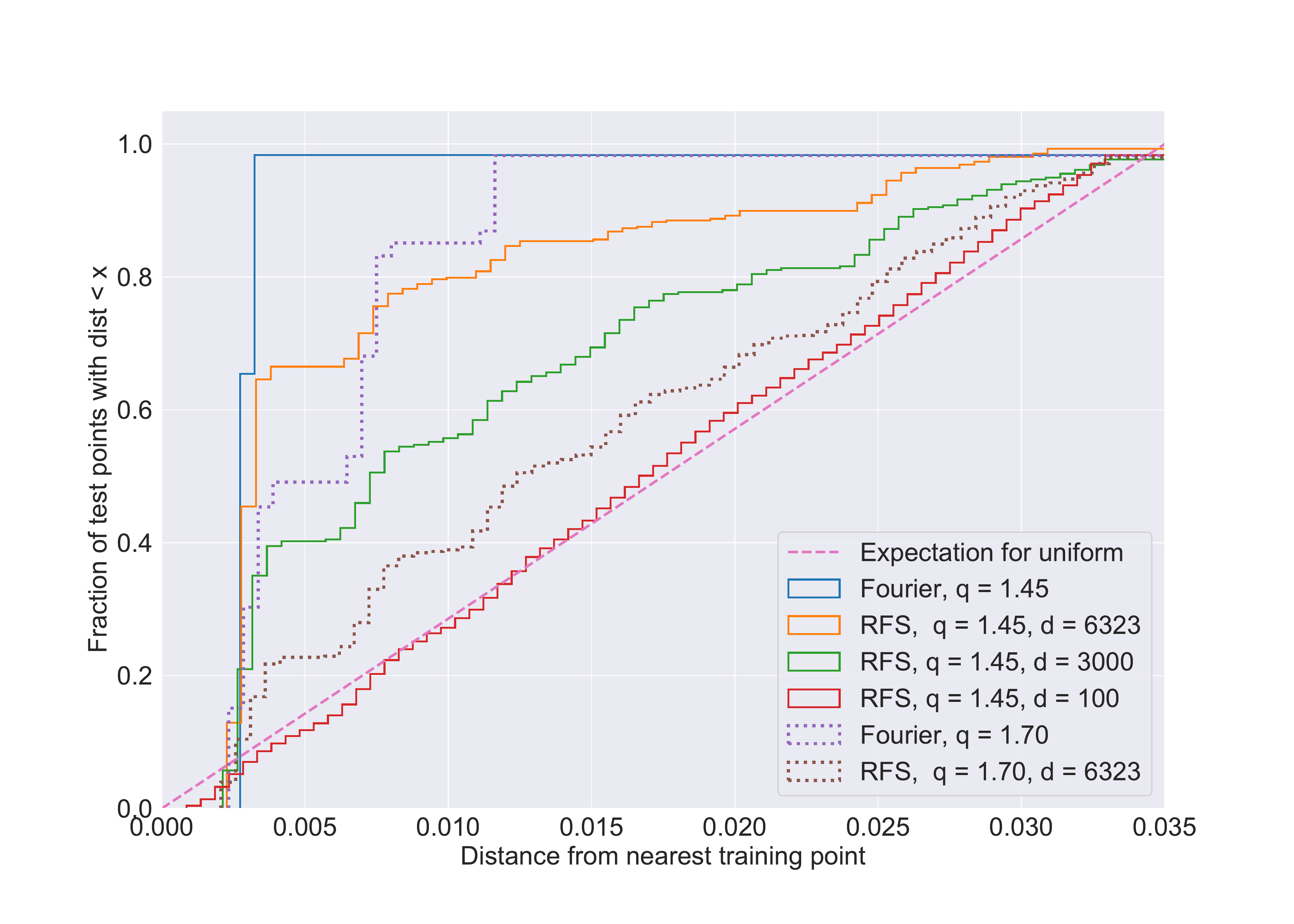}
  \caption{Cumulative histogram of distances of misclassified test
    points from the training point nearest to them. This figure fixes $n=30$ and $B = 901$.}
  \label{fig:cdf}
\end{figure}

\bibliography{references}

\newpage
\appendix
\onecolumn

\section{Expressions for learned coefficients and critical lobe}
\subsection{Expressions for the learned co-efficients: Proof for Lemma~\eqref{lem:coeffs}}

The original optimization problem in Equation~\eqref{def:minl2} can be rewritten by changing coordinates to the variable $\hat{\boldzeta}_2 = \Sigmabold^{-\frac{1}{2}} \boldsymbol{\coeffs}_2$ as
\begin{equation}
\label{eq:weighted_problem}
\widehat{\boldzeta}_2 = \arg \min_{\boldzeta}  \|\boldzeta\|_2: \sum_{k=1}^\B \hat{\zeta_k} \sqrt{\lambda_k} \phi_k(x_j) = 1 \; \forall j \in [n]
\end{equation}
Let $M = \{2 k n \; \forall k \in [N_A]\}$ be the set of indices of all aliases of the constant function. Now, recognize the following two properties of the regularly spaced Fourier features:
\begin{align*}
& \phi_j (\xtrain) = \sqrt{2} \phi_1 (\xtrain) = \boldsymbol{1}, \; \; \forall j \in M \\
& \langle \phi_j (\xtrain) , \boldsymbol{1} \rangle = 0, \; \; \forall j \notin \{M \cup 1 \} \\
\end{align*}
Hence, placing any amount of weight on the non-alias features will not assist with constraint satisfaction and hence we must have $\hat{\zeta}_j = 0, \forall j \notin \{M \cup 1\}$.
Hence, problem~\eqref{eq:weighted_problem} can be rewritten as follows.
\[
  \widehat{\boldzeta}_2 = \arg \min_{\boldzeta}  \|\boldzeta\|_2: \frac{\zeta_1 \sqrt{\lambda_1}} {\sqrt{2}}  + \sum_{k \in M} \zeta_k \sqrt{\lambda_k} = 1
\]
Now, the summation is only over the true feature and its aliases, and there is only a single constraint instead of one for each training point.
We can consider the simpler problem in $N_A$ dimensions instead of $B$ dimensions by restricting to the indices in $\{M \cup 1 \}$.
\begin{equation}
\label{eq:xi_problem}
\smallcoeffs = \arg \min_{\boldxi}  \|\boldxi\|_2: \sum_{k=1}^{N_A} \xi_k \sqrt{\widetilde{\lambda}_k} = 1 \; \forall j \in [n]
\end{equation}
In the above, we define $\boldsymbol{\widehat{\xi}} \in \R^{N_A}$ as $\hat{\xi}_1 = \hat{\zeta}_1$ and $\hat{\xi}_j = \hat{\zeta}_{2jn}$ for all other $j \in \{2, \ldots, N_A\}$.
Similarly, we define $\tilde{\boldsymbol{\lambda}} \in \R^B$ as: \[
\widetilde{\lambda}_j = \begin{cases}
  \frac{\lambda_1}{2} & j = 1 \\
  \lambda_{2jn} & j \in [2, (\frac{\B-1}{2n})]
\end{cases}
.\]
%
%
With slightly unconventional notation, we denote
\[
\widetilde{\boldlambda}^{\frac{1}{2}} = \left[ \sqrt{\widetilde{\lambda}_1} \ldots \sqrt{\widetilde{\lambda}_\B} \right]
\]
as shorthand.
Now, from the Cauchy-Schwarz inequality, we realize that $ \|\boldxi\|_2 \|\widetilde{\boldlambda}^{\frac{1}{2}}\|_2 \geq |\langle \boldxi, \widetilde{\boldlambda}^{\frac{1}{2}} \rangle| = 1$; the inner product equals $1$ because of the interpolation constraint in Equation~\ref{eq:xi_problem}.
Equality in the Cauchy-Schwarz will occur when $\widehat{\boldxi}_2 = C \widetilde{\boldlambda}^{\frac{1}{2}}$ for some constant $C$.

Solving for $C$ using the interpolating constraint we have that $C = \frac{1}{\|\widetilde{\boldlambda}^{\frac{1}{2}}\|^2}$. Hence, $\widehat{\boldxi} = \frac{\widetilde{\boldlambda}^{\frac{1}{2}}}{\|\widetilde{\boldlambda}^{\frac{1}{2}}\|^2}$.
Mapping back to the $\widehat{\boldzeta}_2$ co-ordinates, we increase the dimension back to $B$.
Then, $\coeffs_2 = \Sigmabold^{\frac{1}{2}} \widehat{\boldzeta}_2$, which yields the desired expression.
\subsection{Proof of Lemma~\ref{lem:k_star} (characterizing $k^*$)}\label{sec:kstarproof}

Recall that the recovered function is given by
\begin{align*}
\widehat{f}(x) = \frac{2a - \sqrt{2}b}{2\sqrt{2}} + \frac{b}{2} \cdot D_{N_A}(n \pi x) .
\end{align*}
It is well-known (see, e.g. Exercise 1.1,~\cite{muscalu2013classical}), that
\begin{align*}
|D_{N_A}(n \pi x)| \leq E(x) := \frac{2}{\pi nx} ,
\end{align*}
Further, denote $c := \frac{2a - \sqrt{2}b}{2\sqrt{2}} =
\frac{a}{\sqrt{2}} - \frac{b}{2}$.

We first prove the upper bound on $k^*$.
Our strategy for doing this is as follows:
Noting that $E(x)$ is monotonically decreasing in $x$, we first find the intersection $\overline{x}$ of $-b E(x)$ with $-c$.
By the definition of the envelope, all lobes that \textit{begin} beyond $\overline{x}$ can never cross the line $y = -c$, and thus $\widehat{f}(x) \geq 0$ for all $x > \overline{x}$.
This means that the index of the first lobe whose left boundary is greater than $\overline{x}$ constitutes an upper bound on $k^*$.

First, solving for the intersection of $-E(x)$ with $-c$, we get
\begin{align*}
- \frac{2b}{n \pi \overline{x}} &= - \frac{2a - \sqrt{2}b}{2\sqrt{2}} \\
\implies \overline{x} &= \frac{2\sqrt{2} \cdot b \cdot 2}{\pi n(2a - \sqrt{2}b)} .
\end{align*}
Now, recall that the $k^{th}$ \textit{negative} lobe of the Dirichlet kernel $D_{N_A}(n \pi x)$ is contained in the interval $[\frac{2k -1}{h(\B,n)}, \frac{2k}{h(\B,n)}]$, where $h(\B,n) := n(N_A + 1/2) = (d -1 + n)/2$.

Therefore,
\begin{align*}
\frac{2k-1}{h(\B,n)} &> \frac{2\sqrt{2} \cdot b \cdot 2}{\pi n(2a - \sqrt{2}b)} \\
\implies k &> \frac{2\sqrt{2} \cdot b \cdot h(\B,n)}{\pi n (2a - \sqrt{2}b)} := \frac{2\sqrt{2} \cdot b \cdot (d - 1 + n) + n(2a - \sqrt{2}b)}{2 \pi n (2a - \sqrt{2} b)} .
\end{align*}
The right hand side of the above matches the expression in Equation~\eqref{eq:kstarupperbound}, and completes the proof of the upper bound on $k^*$.

Now, we prove the sufficient condition for the lower bound $k^* \geq 1$.
For this, we again used specialized properties of the Dirichlet kernel.
We know that the global minimum of the function $D_{N_A}(n\pi x)$ will be located in the \textit{first} negative lobe; therefore, it suffices to prove an upper bound on the (negative) value of the global minimum.
One can show\footnote{\label{foot:dirichlet} See, e.g. the calculation in \url{http://www-personal.umd.umich.edu/~adwiggin/TeachingFiles/FourierSeries/Resources/DirichletKernel.pdf}} that
\begin{align*}
\min_{x} D_{N_A}(n \pi x) \leq -C' \cdot (N_A + 1/2) := -C' \cdot \frac{h(\B,n)}{n} ,
\end{align*}
where $C' \sim 0.4344$ is another universal positive constant that does not depend on $d,n,a,b$.
(The approximation is upto $4$ decimal places.)
Thus, we will have $\widehat{f}(x) < 0$ for some $x \in [\frac{1}{h(\B,n)}, \frac{2}{h(\B,n)}]$ provided that
\begin{align*}
\frac{2a - \sqrt{2}b}{2\sqrt{2}} + -0.4344 \frac{b}{2} \cdot \frac{h(\B,n)}{n} &< 0 \\
\implies 0.2172 \sqrt{2}b \cdot (d - 1 + n) > n (2a - \sqrt{2}b) .
\end{align*}
This matches the expression in Equation~\eqref{eq:kstarlowerbound} and completes the proof of Lemma~\ref{lem:k_star}.

\section{Risk calculations}

In this section, we collect proofs of our calculations for the classification and adversarial risk.

\subsection{Proof of Theorem~\ref{thm:riskabdn}}\label{sec:riskabdnproof}

We start by getting an upper bound on the classification risk for a trigonometric function of the form $f_{(a,b,d,n)}(\cdot)$.
Because $f(\cdot)$ is periodic with period equal to $2/n$, it suffices to evaluate the risk on a single period of length $2/n$.
We first note that $\mathcal{C}(f) \leq \mathcal{C}_{\mathsf{adv}}(f,2\pi/h(\B,n))$.
Thus, it suffices to obtain an upper bound for the latter.
We use the lobe structure of the Dirichlet kernel to characterize some important properties of the zero-crossings of the function $f$ and their neighborhoods, in the following lemma.

\begin{lemma}
 \label{lem:overlaps}

 For the choice of perturbation $\epsilon = \frac{2\pi}{h(\B,n)}$, the padding around zero-crossings are given by:
\begin{align*}
   \widetilde{r}_{k-1} = \widetilde{l_k} \text{ and } \widetilde{l}_1 = 0
\end{align*}
for all lobes indexed by $k \leq k^*$.
\end{lemma}
\begin{proof}
The distance between the locations of minima of any two successive lobes is exactly $\frac{2\pi}{h(\B,n)}$. We know that $d_k = \widetilde{l}_k - \widetilde{r}_{k-1}$ cannot be larger than this. We recall that
\begin{align*}
\widetilde{l}_k = \begin{cases}
\widetilde{r}_{k-1},& d_k < 2\epsilon,\\
l_k - \epsilon, & \text{otherwise.}
\end{cases}
\end{align*}
Our choice of $\epsilon$ implies that we will always hit the first case. Further, $l_1 < \frac{2\pi}{h(\B,n)}$ must be true because $\frac{2\pi}{h(\B,n)}$ is the co-ordinate of the end of the first lobe. Since $\widetilde{l_1} = \max(l_1 - \epsilon, 0)$, we will always choose  $0$.
\end{proof}

Then, plugging Lemma~\ref{lem:overlaps} directly into the expression for the adversarial risk as a function of the zero-crossings gives us
\begin{align*}
\mathcal{C}_{\mathsf{adv}}(f,2\pi/h(\B,n)) = k^* \cdot \frac{\frac{4 \pi}{(h(\B,n))} + \frac{2 \pi}{h(\B,n)}}{1/n} &\leq 2 k^* \cdot \frac{4 \pi n}{(h(\B,n))} .
\end{align*}
Substituting the upper bound for $k^*$ from Lemma~\ref{lem:k_star} yields Equation~\eqref{eq:risk_abdn} and Equation~\eqref{eq:advsmall_abdn}, and completes the proof of upper bounds on classification and adversarial risk to perturbations on the order of $1/h(\B,n)$.

We now exactly characterize the adversarial risk for perturbation $\epsilon = 2/n$.
Note that there is a sharp phase transition in the behavior depending on the value of $k^*$:
\begin{enumerate}
\item If $k^* \geq 1$, we know that the global minima of $f(\cdot)$ are strictly negative.
Then, the periodicity of the function $f(\cdot)$ directly implies that any perturbation of magnitude at least $2/n$ would find this negative-valued point.
\item On the other hand, if $k^* = 0$, it means that $f(\cdot) \geq 0$ for all $x$ and the adversarial error will always be equal to $0$.
\end{enumerate}

This directly tells us that if Equation~\eqref{eq:kstarlowerbound} holds, the adversarial risk will be $1$; otherwise, it will be equal to $0$.
This completes the proof of the theorem.
\qed

\subsection{Proof of Corollary~\ref{cor:riskbilevel}}

In this section, we prove Corollary~\ref{cor:riskbilevel} as a consequence of Theorem~\ref{thm:riskabdn} for the bilevel ensemble.

First, we characterize the classification risk (and adversarial risk to perturbations of $\epsilon = 2\pi/h(\B,n))$).
Using the expressions from Lemma~\ref{lem:coeffs}, simple algebra gives us
\begin{subequations}
\begin{align}
a &= \frac{1}{1 + n^{q - 1}} \text{ and } \label{eq:a} \\
b &= \frac{1}{n^{p - q} + n^{p - 1}} \label{eq:b}.
\end{align}
\end{subequations}
For the case $q > 1$, we have
\begin{align*}
a &\asymp n^{-(q-1)} \\
b &\asymp n^{-(p - q)} .
\end{align*}
Then, our upper bound on the classification risk was
\begin{align*}
\frac{\sqrt{2} C b \cdot h(\B,n) + n(2a - \sqrt{2}b)}{2h(\B,n) (2a - \sqrt{2}b)} &= \frac{b}{2(2a - \sqrt{2}b)} + \frac{n}{h(\B,n)} \\
&= \frac{n^{q - p}}{2(2 - \sqrt{2} n^{q - p})} + n^{-(p-1)} \\
&\leq C n^{-(p-q)} + n^{-(p-1)} \\
&\leq C n^{-(p-q)}
\end{align*}
for some constant $C$ that is independent of $n$.
This verifies the statement of Equation~\eqref{eq:risk_bilevel}, as well as Equation~\eqref{eq:advsmall_bilevel}.

Next, we characterize the adversarial risk to perturbations of $\epsilon = 2/n$.
Observe that we needed the condition
\begin{align*}
2\sqrt{2} \cdot 0.2172 \cdot b \cdot (\B  - 1  + n) > n (2a - \sqrt{2}b)
\end{align*}
to hold.
Substituting Equations~\eqref{eq:a} and~\eqref{eq:b} tells us that the condition holds if and only if we have
\begin{align*}
t(n) := \frac{2\sqrt{2} \cdot 0.2172 \cdot (n^{p+1} - n + n^2 + n^{p+q} - n^q + n^{q+1})}{n^{p-q}+n^{p+1} -\sqrt{2}n^2 + \sqrt{2} n^{q+1}} > 1 .
\end{align*}
For the case $q > 1$, observe that $t(n) \geq \frac{0.2172 \sqrt{2}}{2 + \sqrt{2}} \cdot n^{q - 1} > 1$ if
\begin{align*}
n \geq n_0(q) := \left(\frac{2 + \sqrt{2}}{0.2172 \cdot 2\sqrt{2}}\right)^{\frac{1}{q - 1}} \approxeq (5.55)^{\frac{1}{q-1}} .
\end{align*}
On the other hand, for $q < 1$, observe that $t(n) \leq \frac{2\sqrt{2} \cdot 0.2172 (1 + n^{-(1 - q)} + n^{-(p-q)})}{1 - \sqrt{2} n^{- (p-1)}}$, and since all of the exponents in $n$ are strictly negative, we have $t(n) \leq 0.8 < 1$ for large enough $n$.
This completes the proof of the corollary.
\qed

\section{Heuristic for randomly spaced data}\label{sec:randomlyspaced}

In this section, we provide a heuristic calculation to show that when
the training data is drawn uniformly-at-random instead of being
regularly spaced, we will get
\begin{align}\label{eq:advriskrandom}
\E\left[\mathcal{C}_{\mathsf{adv}}(f,1/n)\right] \geq \left(1  - \frac{1}{e}\right) ,
\end{align}
as long as Equation~\eqref{eq:kstarlowerbound} holds. Above, the
expectation is taken over the randomness in the training points in
addition to the randomness in the test points.

The calculation is heuristic because of the following assumption:

\begin{center}
\textit{When the condition of Equation~\eqref{eq:kstarlowerbound}
  holds, each training point, regardless of its location, generates
  two \textit{distinct} adversarial examples in its immediate vicinity
  due to the Gibbs phenomenon. Setting $\epsilon_n = 1/n$, we get $\inf_{x' \in [x_i - \epsilon_n, x_i]} f(x') < 0$ and $\inf_{x'
    \in [x_i, x_i + \epsilon_n]} f(x') < 0$. }
\end{center}

This assumption is borne out by the visualizations of the recovered
function from randomly spaced training data in
Figure~\ref{fig:differentfamilies}.

\begin{figure}[htpb]
\begin{subfigure}[]{0.49\linewidth}
  \centering
  \includegraphics[width=\textwidth]{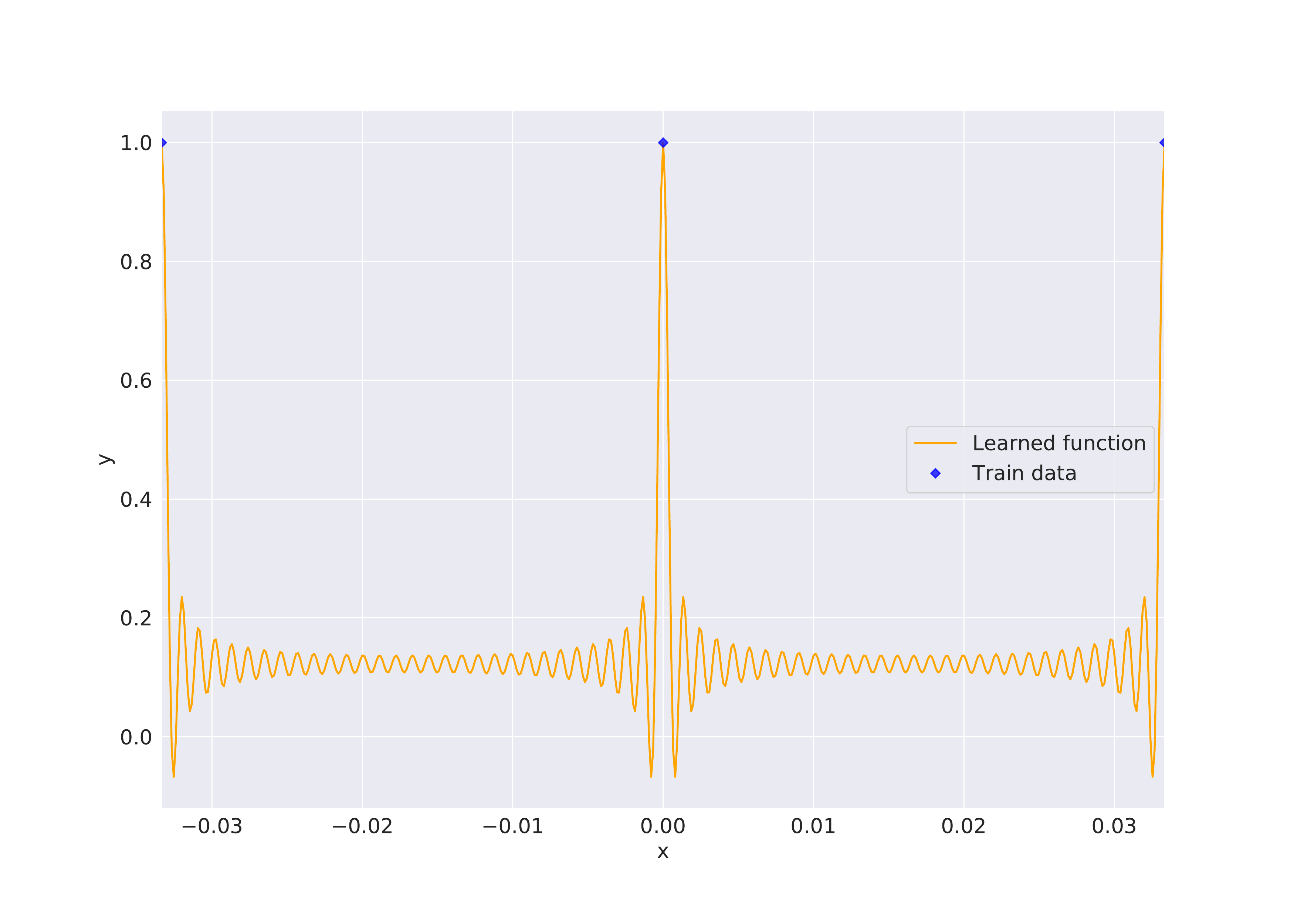}
 \caption{Fourier, regularly spaced }
 \label{subfig:fp_reg_four}
 \end{subfigure}
 \begin{subfigure}[]{0.49\linewidth}
   \centering
   \includegraphics[width=\textwidth]{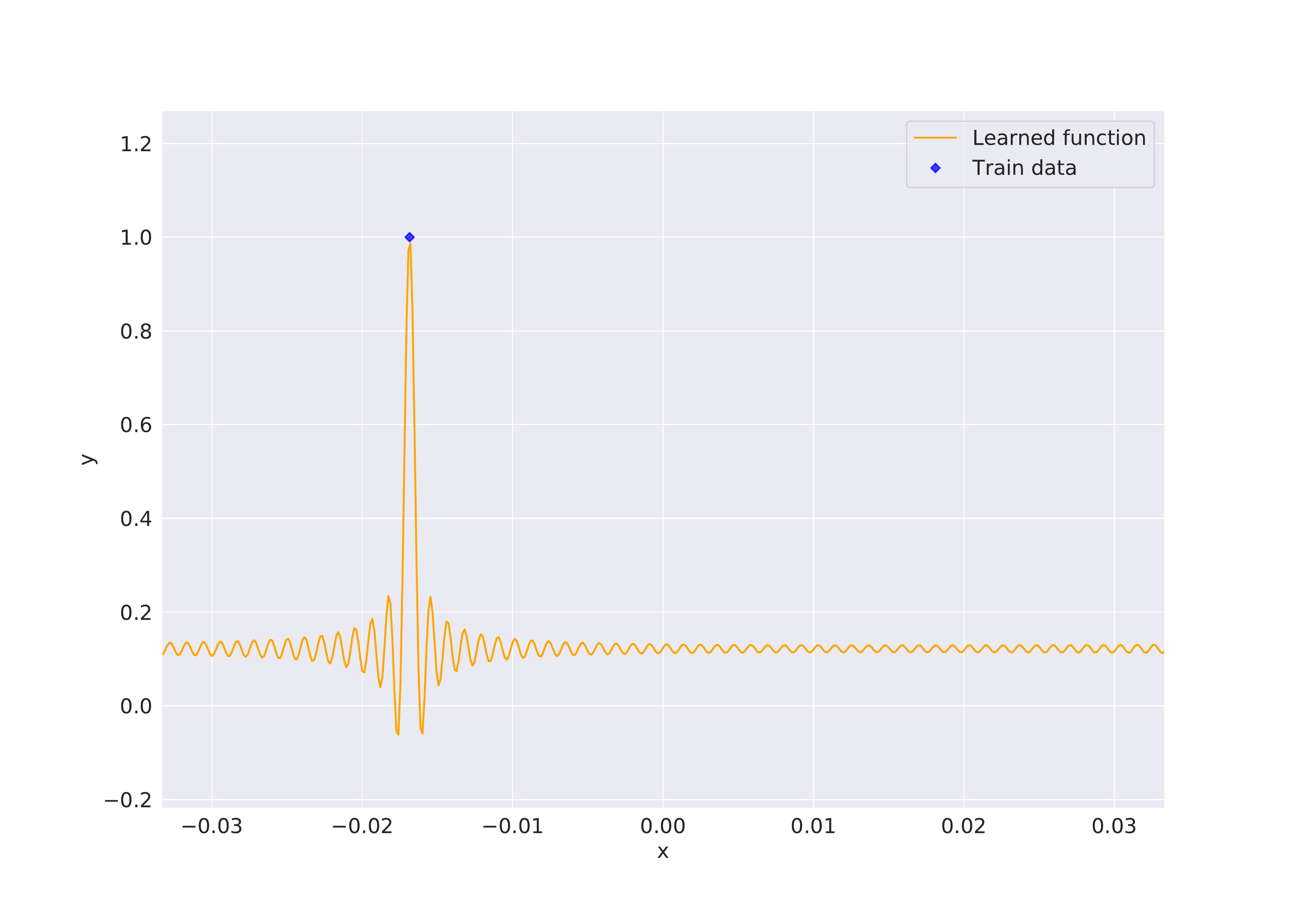}
 \caption{Fourier, randomly sampled}
 \label{subfig:fp_random_four}
 \end{subfigure}
 \caption{Zoomed-in recovered function plots for different samplings
   of training data. Throughout, $n=60, p = 2, r = 0, q = 1.45$, and all
  parameters vary according to the bilevel scalings in
  Definiton~\eqref{def:bilevel_covariance}. In each of the cases, the
  learned function can be seen to cross the x-axis near the training
  point due to Gibbs-ean effects; however, the interval of $x$ for which the
  function's output is negative is very small. This causes the
  function to do well on average but do poorly in the face of an
  adversary who has the ability to search for these small regions of
  misclassification.}
 \label{fig:differentfamilies}
\end{figure}

From this assumption, we will use a balls-and-bins argument to
heuristically approximate the adversarial error in the
following fashion. A test point cannot be pushed within $\epsilon_n$ to
a training point if and only if none of those training points happened
to land within the size $2\epsilon_n$ neighborhood around the test
point. But each training point is assumed to have been drawn
independently at random from a uniform distribution from $[-1,+1]$ and
so the probability of this happening for any single training point is
$1 - \frac{1}{n}$. This has to happen for each of the training points,
which gives $(1-\frac{1}{n})^n$ by independence, and as $n \rightarrow
\infty$, this tends to $e^{-1}$. Consequently, the probability of not
having an adversarial error is $e^{-1}$ and so the probability of
having an adversarial error is $1-e^{-1} \approx 0.623$.

This yields corresponding predictions for the adversarial risk to
Equations~\eqref{eq:advmed_abdn} and~\eqref{eq:advmed_bilevel}, with
$1$ replaced by $\approx 0.623$. This calculation demonstrates that
while the adversarial risk is less for randomly spaced training data
than for regularly spaced training data, it still persists at
\textit{at least} a large constant fraction value as $n \to \infty$
under various overparameterized ensembles.

\section{Proof of Theorem~\ref{thm:rfs_large_d}}\label{sec:rfsproof}

Denote $\boldsymbol{W} \in \R^{B \times d}$ as $\boldsymbol{W} = [\boldsymbol{w_1}, \boldsymbol{w_2}, \ldots \boldsymbol{w_d}]$.
Then $\phirfs(\xtrain) = \phi(\xtrain) \boldsymbol{W}$.
First, we state the following important lemma that follows as an immediate consequence of Random Matrix Theory: see Equation (6.12) in \cite{wainwright2019high}.
\begin{lemma}
  \label{lem:random_matrix}
  Under the Definition~\ref{def:rfs}, with probability $1 - 2e^{-d \delta^2/2}$,
\begin{equation}
  \frac{\norm{\frac{1}{d} \boldsymbol{W} \boldsymbol{W}^T - \Sigmabold}}{\norm{\Sigmabold}} \leq 2 \sqrt{\frac{B}{d}} + 2\delta + \left(\sqrt{\frac{B}{d}} + \delta \right)^2
\end{equation}
\end{lemma}
For a fixed $B$ and growing $d$, the sequence of empirical covariance matrices that varies with $d$, $\{\boldsymbol{W} \boldsymbol{W}^T\}$ converges to $\Sigmabold$. Now, recall that the solution to the min-$L_2$ interpolation problem can be written using the formula for the pseudoinverse.
\begin{equation}
  \label{eq:coeffs_two}
  \coeffs_2 = \Sigmabold^{\frac{1}{2}} \left(\phi(\xtrain) \Sigmabold^{\frac{1}{2}} \right)^\dagger \ytrain = \Sigmabold  \phi(\xtrain)^T (\phi(\xtrain) \Sigmabold \phi(\xtrain)^T)^{-1} \ytrain
\end{equation}
Similarly, the RFS coefficients can be written as:
\[
\coeffsRFS = \boldsymbol{W}^T \phi(\xtrain)^T (\phi(\xtrain) \boldsymbol{W} \boldsymbol{W}^T \phi(\xtrain)^T)^{-1} \ytrain
.\]
Writing the weights in the original Fourier space, we have
\begin{equation}
  \label{eq:coeffs_rfs}
  \coeffs^{RFS} = \boldsymbol{W} \coeffsRFS = \boldsymbol{W} \boldsymbol{W}^T \phi(\xtrain)^T (\phi(\xtrain) \boldsymbol{W} \boldsymbol{W}^T \phi(\xtrain)^T)^{-1} \ytrain
\end{equation}
Comparing $\coeffs^{RFS}$ in Equation~\ref{eq:coeffs_rfs} to $\coeffs_2$ in Equation~\ref{eq:coeffs_two}, we recognize that each $\Sigmabold$ has been replaced by the empirical covariance $\boldsymbol{W} \boldsymbol{W}^T$. Denote $\boldsymbol{M} = \phi(\xtrain)$ and write the Euclidean-norm of the difference in coefficients as:
\begin{align*}
  \norm{\coeffs_2 - \alphaRFS} &= \norm{\Sigmabold \boldsymbol{M}^\top (\boldsymbol{M} \Sigmabold \boldsymbol{M}^\top)^{-1} \ytrain  - \boldsymbol{W} \boldsymbol{W}^T \boldsymbol{M}^\top (\boldsymbol{M} \boldsymbol{W} \boldsymbol{W}^T \boldsymbol{M}^\top)^{-1} \ytrain}  \\
			       & \leq \underbrace{\norm{ \Sigmabold \boldsymbol{M}^\top (\boldsymbol{M} \Sigmabold \boldsymbol{M}^\top)^{-1} \ytrain  - \frac{1}{d} \boldsymbol{W} \boldsymbol{W}^\top \boldsymbol{M}^\top (\boldsymbol{M} \Sigmabold \boldsymbol{M}^\top)^{-1} \ytrain} }_{T_1} + \\ 
			       & \underbrace{\norm{\boldsymbol{W} \boldsymbol{W}^T \boldsymbol{M}^\top (\boldsymbol{M} \boldsymbol{W} \boldsymbol{W}^T \boldsymbol{M}^\top)^{-1} \ytrain -  \frac{1}{d} \boldsymbol{W} \boldsymbol{W}^\top \boldsymbol{M}^\top (\boldsymbol{M} \Sigmabold \boldsymbol{M}^\top)^{-1} \ytrain} }_{T_2}
\end{align*}
The second step above was obtained by adding and subtracting terms and using the triangle inequality. Now, we bound each of the terms independently.

\noindent \textbf{Bounding $\mathbf{T_1}$:}
Let $\boldsymbol{\Delta} = (\frac{1}{d} \boldsymbol{W} \boldsymbol{W}^\top - \Sigmabold)$ be the error in covariance estimate. Below, we use the spectral norm for matrices.

\begin{equation}
  \label{eq:t1_split}
  \norm{ (\Sigmabold - \frac{1}{d} \boldsymbol{W} \boldsymbol{W}^\top)  (\boldsymbol{M}^\top (\boldsymbol{M} \Sigmabold \boldsymbol{M}^\top)^{-1} \ytrain)}  \leq \norm{ \Sigmabold - \frac{1}{d} \boldsymbol{W} \boldsymbol{W}^\top}  \; \norm{\boldsymbol{M}^\top (\boldsymbol{M} \Sigmabold \boldsymbol{M}^\top)^{-1} \ytrain}
\end{equation}
From Lemma~\ref{lem:random_matrix}, we have that $\norm{\boldsymbol{\Delta}} \leq \bigO\left(\sqrt{\frac{B}{d}}\right) \norm{\Sigmabold} = n^{\frac{3p}{2} - q - \frac{T}{2}}$. To obtain the scalings with $n$, we used from Definition~\ref{def:bilevel_covariance} that $B = n^p, d = n^T, \norm{\Sigmabold} = n^{p-q}$.

For the second term in Equation~\ref{eq:t1_split}, we recall from Equation~\ref{eq:coeffs_two} that $\coeffs_2 = \Sigmabold \boldsymbol{M}^\top (\boldsymbol{M} \Sigmabold \boldsymbol{M}^\top)^{-1} \ytrain$. And we have a closed form expression for this from Lemma~\ref{lem:coeffs}. Hence, we can write the second term as
\begin{equation}
  \norm{\boldsymbol{M}^\top (\boldsymbol{M} \Sigmabold \boldsymbol{M}^\top)^{-1} \ytrain}  = \norm{\Sigmabold^{-1} \coeffs_2}  \leq
  \norm{ \Sigmabold^{-1}} \norm{ \begin{bmatrix} \frac{\sqrt{2} \lambda_1}{\lambda_1 + \lambda_L \frac{B-1}{n}} \\ 0 \\ \vdots \\ \frac{2 \lambda_L}{\lambda_1 + \lambda_L \frac{B-1}{n}} \\ \vdots  \end{bmatrix}}
\end{equation}
We know that $\norm{\Sigmabold^{-1}} = \frac{1}{\lambda_d} = \frac{1}{1 - n^{-q}}$. In the vector $\coeffs_2$ above, there are $0$s in the positions of all non-alises of $\ytrain$ and the weights of the first true feature and the $N_A$ aliases are as specified. We can write the norm of the vector as
\begin{align*}
  \norm{\coeffs_2} & = \frac{1}{\lambda_1 + \lambda_L (\frac{B-1}{n})} \sqrt{2 \lambda_1^2 + 4 \left(\frac{B-1}{n} \right) \lambda_L^2}  \\
 & = \frac{1}{\gamma \frac{n-1}{n} + \frac{1}{n}} \sqrt{\gamma^2 + \frac{B}{n} (1 + \gamma^2 - 2 \gamma)} \\
 & = \bigO \left(\frac{n^{-q} + n^{\frac{p-1}{2}}}{n^{-q} + n^{-1}} \right) = \bigO (n^{\frac{p+1}{2}})
\end{align*}
Hence, $T_1$ scales as  $\bigO(\frac{n^{2p + \frac{1}{2} - q - \frac{T}{2}}}{1 - n^{-q}})$. This uses the assumption $q > 1$.

\noindent \textbf{Bounding $\mathbf{T_2}$:} Define $\boldsymbol{G} = \boldsymbol{M} \Sigmabold \boldsymbol{M}^\top, \boldsymbol{S} = \boldsymbol{M} \boldsymbol{\Delta} \boldsymbol{M}^\top$. Then, write $T_2$ as:
\begin{align*}
T_2  & = \norm{\frac{1}{d} \boldsymbol{W} \boldsymbol{W}^T \boldsymbol{M}^\top \left(\frac{1}{d} \boldsymbol{M} \boldsymbol{W} \boldsymbol{W}^T \boldsymbol{M}^\top \right)^{-1} \ytrain -  \frac{1}{d} \boldsymbol{W} \boldsymbol{W}^\top \boldsymbol{M}^\top (\boldsymbol{M} \Sigmabold \boldsymbol{M}^\top)^{-1} \ytrain}  \\
&= \norm{\frac{1}{d} \boldsymbol{W} \boldsymbol{W}^T \boldsymbol{M}^\top \left[\left(\frac{1}{d} \boldsymbol{M} \boldsymbol{W} \boldsymbol{W}^T \boldsymbol{M}^\top \right)^{-1} - (\boldsymbol{M} \Sigmabold \boldsymbol{M}^\top)^{-1}\right] \ytrain}  \\
&= \norm{\frac{1}{d} \boldsymbol{W} \boldsymbol{W}^T \boldsymbol{M}^\top \left[ (\boldsymbol{G} + \boldsymbol{S})^{-1} - \boldsymbol{G}^{-1} \right] \ytrain} \numberthis \label{eq:t1_first_align}
\end{align*}
Consider the infinite-series expansion of the inverse of the sum of two matrices:
\begin{equation}
\label{eq:matrix_identity}
(\boldsymbol{G} + \boldsymbol{S})^{-1} =  (\boldsymbol{I} + \boldsymbol{G}^{-1} \boldsymbol{S})^{-1} \boldsymbol{G}^{-1} = \boldsymbol{G}^{-1} + \boldsymbol{G}^{-2} \boldsymbol{S} + \boldsymbol{G}^{-3} \boldsymbol{S}^2 \cdots
\end{equation}
We can rewrite Equation~\eqref{eq:t1_first_align} as follows, using Equation~\eqref{eq:matrix_identity} and the simple facts about norms for matrices: $\norm{A + B} \leq \norm{A} + \norm{B}$ and $\norm{A B} \leq \norm{A} \norm{B}$.

\begin{align*}
T_2 & = \frac{1}{d} \norm{\boldsymbol{W} \boldsymbol{W}^\top \boldsymbol{M}^\top} \norm{ \boldsymbol{G}^{-2} \boldsymbol{S} + \boldsymbol{G}^{-3} \boldsymbol{S}^2 \cdots}  \norm{\boldsymbol{y}} \\
& \leq \frac{1}{d} \norm{\boldsymbol{W} \boldsymbol{W}^\top \boldsymbol{M}^\top} \left[ \sum_{i=1}^\infty \norm{\boldsymbol{G}^{-{i+1}} \boldsymbol{S}^i}  \right] \norm{\boldsymbol{y}} \\ 
& \leq \frac{1}{d} \norm{\boldsymbol{W} \boldsymbol{W}^\top \boldsymbol{M}^\top} \left[\sum_{i=1}^\infty \norm{\boldsymbol{G}^{-1}}^{i+1} \norm{\boldsymbol{S}}^{i} \right] \norm{\boldsymbol{y}} \numberthis \label{eq:T2_inter}
\end{align*}
Notice that the series above adopts the form of a geometric series with first term $a = \norm{\boldsymbol{G}^{-1}}^2 \norm{\boldsymbol{S}}$ and multplicative factor $r = \norm{\boldsymbol{G}^{-1}} \norm{\boldsymbol{S}}$. Hence, since $r = n^{-q - T}$ (see calculation below), for large enough $n$, $r < 1$. Then the series evaluates to  \[
  \frac{\norm{\boldsymbol{G}^{-1}}^2 \norm{\boldsymbol{S}}}{1 - \norm{\boldsymbol{G}}^{-1} \norm{\boldsymbol{S}}}
.\]
Plugging back in Equation~\eqref{eq:T2_inter}, we have:
\[
T_2  \leq \frac{1}{d} \norm{\boldsymbol{W} \boldsymbol{W}^\top}  \norm{ \boldsymbol{M}^\top} \left[ \frac{\norm{\boldsymbol{G}^{-1}}^2 \norm{\boldsymbol{S}}}{1 - \norm{\boldsymbol{G}^{-1}} \norm{\boldsymbol{S}}} \right] n
.\]
From Lemma~\ref{lem:random_matrix}, we know that $\frac{1}{d} \norm{\boldsymbol{W} \boldsymbol{W}^\top}$ scales as $\norm{\Sigmabold} = n^{p-q}$. Since $\boldsymbol{M}$  can be written as the $n$-dimensional DFT matrix horizontally stacked  $\frac{B}{n}$ times, it is easy to check that $\boldsymbol{M} \boldsymbol{M}^T = B I_n$; hence  $\norm{\boldsymbol{M}^\top} = \sqrt{B}$, and $\sigma_{min}(\boldsymbol{M}) = \sqrt{B}$.

Recognize that $\norm{\boldsymbol{S}} \leq \norm{\boldsymbol{M}}^2 \norm{\boldsymbol{\Delta}} = n^{p} n^{\frac{3p}{2} - q - \frac{T}{2}} = n^{\frac{5p}{2} - q - \frac{T}{2}}$. Let $\sigma_{min} (\cdot)$ denote the minimum singular value of a matrix; then, using the fact that the minimum singular value of a product of matrices is always lower-bounded by the product of their individual singular values, we have:
\begin{equation}
\label{eq:scaling_B}
\sigma_{min} (\boldsymbol{G}) \geq \sigma_{min} (\boldsymbol{M})^2 \lambda_d = n^{p} \times (1 - n^{-q}) = n^{p} - n^{p - q}
\end{equation}
Then, \[
  \norm{\boldsymbol{G}^{-1}} = \frac{1}{\sigma_{min} (\boldsymbol{G})} \leq \frac{1}{n^{p} - n^{p - q}} = \bigO(n^{-p})
.\]
Plugging into the formula for the series, we get
\begin{equation}
\label{eq:}
\frac{\norm{\boldsymbol{G}^{-1}}^2 \norm{\boldsymbol{S}}}{1 - \norm{\boldsymbol{G}^{-1}} \norm{\boldsymbol{S}}} = \bigO \left(\frac{n^{\frac{p}{2} - q - \frac{T}{2}}}{1 - n^{\frac{3p}{2}-q-\frac{T}{2}}} \right).
\end{equation}
Combining all terms for $T_2$, we obtain that
 \begin{equation}
\label{eq:T2_scaling}
T_2 \leq \bigO(n n^{p-q} n^{\frac{p}{2}} \frac{n^{\frac{p}{2}} - q - \frac{T}{2}}{1 - n^{\frac{3p}{2}-q-\frac{T}{2}}}) = \bigO \left(\frac{n^{2p - 2q - \frac{T}{2} + 1}}{1 - n^{\frac{3p}{2} - q - \frac{T}{2}}}\right)
\end{equation}

\noindent \textbf{From coefficients to function convergence}

Define $e(x) = \langle \coeffs^{RFS}, \phi(x) \rangle - \langle \coeffs, \phi(x) \rangle = \sum_{i=1}^n (\coeffs^{RFS} - \coeffs_2) \phi_i(x)$.
\begin{align*}
\label{eq:ex_bound}
  \left|e(x) \right| &= \left| \sum_{i=1}^B (\coeffs^{RFS} - \coeffs_2) \cos(2\pi i x) \right|\\
       & \le \| \coeffs^{RFS} - \coeffs_2 \|_\infty \left|\sum_{i=1}^B \cos(2 \pi i x) \right| \\
       & \le \| \coeffs^{RFS} - \coeffs_2 \|_2 B \\
       & =  \bigO \left(B \cdot \frac{n^{2p + \frac{1}{2} - q - \frac{T}{2}}}{1 - n^{-q}}\right) +  \bigO \left(B \cdot \frac{n^{2p - 2q - \frac{T}{2} + 1}}{1 - n^{\frac{3p}{2} - q - \frac{T}{2}}}\right) \\
       & = \bigO \left(\frac{n^{3p + \frac{1}{2} - q - \frac{T}{2}}}{1 - n^{-q}}\right) + \bigO \left(\frac{n^{3p - 2q - \frac{T}{2} + 1}}{1 - n^{\frac{3p}{2} - q - \frac{T}{2}}}\right)
\end{align*}
Hence when $T > \max(6p + 1 -2q, 6p - 4q + 2, 3p - 2q) = \max(6p+1-2q, 6p - 4q + 2)$, the error above can be confirmed to decay to $0$. When $q > \frac{1}{2}$, the $\max$ can be simplified to $(6p + 1 - 2q)$.

\vspace{1em}

\noindent \textbf{Adversarial risk:} From classical results on the Dirichlet kernel (see Footnote~\ref{foot:dirichlet}), we identify that the size of the Gibbs-dip scales as $\frac{B}{n}$. Hence, at the location of the gibbs-dip,
\begin{align}
  \langle \coeffs^{RFS}, \phi(x) \rangle & \leq \langle \coeffs, x \rangle + \|e(x)\|_\infty \\ 
					 & \leq \frac{2a - b\sqrt{2}}{2\sqrt{2}} 1 - \frac{B b}{n} + \bigO \left(\frac{n^{3p + \frac{1}{2} - q - \frac{T}{2}}}{1 - n^{-q}}\right) + \bigO \left(\frac{n^{3p - 2q - \frac{T}{2} + 1}}{1 - n^{\frac{3p}{2} - q - \frac{T}{2}}}\right) \label{eq:adv_risk_rfs}
.\end{align}
From the assumption on $T$ for the theorem, it follows that  $T > 3p - 2q$. Hence the constant is dominant in the denominator for the final two terms in the expression above. Substituting the scalings for $a \asymp n^{1 - q}$ and $b \asymp n^{1-p}$ which hold when $q > 1$ and ignoring constants, the RHS of Equation~\eqref{eq:adv_risk_rfs} is simplified as:
\begin{equation}
\label{eq:}
\bigO \left(n^{1-q} - n^{1-p} - 1 + n^{3p + \frac{1}{2} - q - \frac{T}{2}} + n^{3p - 2q - \frac{T}{2} + 1} \right)
\end{equation}
Since $q > 1$ and  $p > 1$ is always true, the first two terms decay to $0$ in the limit.  If we ensure that the limit is negative, then the classifier must make a mistake in the limit since the true label is positive. To ascertain this, the final two terms must decay as well, which gives us the condition on $T$ as
 \begin{equation}
\label{eq:T_cond}
T > \max(6p + 1 - 2q, 6p + 2 - 4q) = 6p + 1 - 2q
\end{equation}

\noindent \textbf{Classification risk:} Here we reproduce the proof-strategy of the Fourier case, whilst accounting for the difference in function values $e(x)$.
\begin{align*}
  \langle \coeffs^{RFS}, \phi(x) \rangle  & \geq \langle \coeffs, x \rangle - \|e(x)\|_\infty \\ 
					  & \geq c -\|e(x) \| _\infty + \frac{b}{2} D_{N_A} (n \pi x)
\end{align*}
In the above, $c = \frac{2a - b\sqrt{2}}{2\sqrt{2}}$ and $E(x) = \frac{2b}{n\pi x}$ is the envelope of the Dirichlet kernel, as in the Fourier case. Now, we want to find a condition on $T$ to ensure that the RHS above is positive for almost the entire domain for all $q > 1$.

Since  $E(x)$ is a monotonically decreasing function,  $- E(x)$ is monotonically increasing and consequently, so is the RHS of the expression in the previous display. Hence, if we find the intersection  $\overline{x}$ of the RHS with $0$, then for all  $x > \overline{x}$, we are guaranteed that the RHS will be positive. Setting the RHS to $0$, it is found that
 \begin{equation}
\label{eq:x_bar_RFS}
\overline{x} = \frac{2b}{n \pi (c - \|e(x)\|_\infty)} \geq \frac{2b}{n \pi \tilde{c}}
\end{equation}
We define $\tilde{c}$ by substituting the bound for $\|e(x)\|_\infty$. We simplify the bound in a manner using the same argument as the adversarial risk, since $T > 3p - 2q$ follows from the assumption, and $3p + \frac{1}{2} - q - \frac{T}{2} > 3p - 2q - \frac{T}{2} + 1$.
\begin{align*}
  \tilde{c} & = c - \bigO \left(\frac{n^{3p + \frac{1}{2} - q - \frac{T}{2}}}{1 - n^{-q}}\right) + \bigO \left(\frac{n^{3p - 2q - \frac{T}{2} + 1}}{1 - n^{\frac{3p}{2} - q - \frac{T}{2}}}\right) \\
 &  \asymp c - n^{3p + \frac{1}{2} - q - \frac{T}{2}}
\end{align*}
Hence, for a single period in $(0, \frac{1}{n})$, any misclassifications are guaranteed to be contained in the interval $(0, \overline{x})$. We want to ensure that the size of this interval decays to $0$ as  $n \to \infty$.
\begin{align*}
  \lim_{n \to \infty} \frac{\overline{x}}{\frac{1}{n}} & = \frac{2b}{n \pi \tilde{c}} \times n  \\
						       & = \frac{2b}{\tilde{c}} =  \frac{n^{1 - p}}{n^{1-q} - n^{1-p} - n^{3p + \frac{1}{2} - q - \frac{T}{2}}}
.\end{align*}
For the above limit to evaluate to $0$, it is sufficient to ensure:
\begin{enumerate}
  \item $1 - p < 1 - q \implies q < p$: This is the same condition as the deterministic Fourier case.
  \item $1 - q > 3p + \frac{1}{2} - q - \frac{T}{2} \implies T > 6p - 1$: This ensures that the survived part of the signal $\tilde{c}$ is always positive with high probability.
\end{enumerate}

This completes the proof.
\qed

\end{document}